\documentclass[journal]{IEEEtran}
\usepackage{amsmath}
\usepackage{amsthm}
\usepackage{amssymb}
\usepackage{algorithm}
\usepackage{array}
\usepackage{algpseudocode}

\usepackage{booktabs}
\usepackage{caption}
\usepackage{cite}
\usepackage{enumitem}
\usepackage{makecell}   
\usepackage{graphicx}
\usepackage{hyperref} 
\usepackage{multirow}
\usepackage{rotating} 
\usepackage{ragged2e} 
\usepackage{stfloats}
\usepackage{soul}
\usepackage{subfigure}
\usepackage{tabularx}
\usepackage{textcomp}
\usepackage{times}
\usepackage{url}
\usepackage{verbatim}
\usepackage{xcolor}

\newtheorem{theorem}{Theorem}

\newtheorem{definition}{Definition}

\theoremstyle{remark}
\newtheorem{remark}[theorem]{Remark}
\newlist{tabitemize}{itemize}{1} 
\setlist[tabitemize]{label=\textbullet, nosep, leftmargin=*, 
                     before={\begin{minipage}[t]{\linewidth}\RaggedRight},
						after={\end{minipage}}}

\begin{document}

\title{Dynamic Influence Tracker: Measuring Time-Varying Sample Influence During Training}

\author{Jie Xu, Zihan Wu
}

\markboth{Journal of \LaTeX\ Class Files,~Vol.~14, No.~8, August~2021}%
{Shell \MakeLowercase{\textit{et al.}}: A Sample Article Using IEEEtran.cls for IEEE Journals}


\maketitle

\begin{abstract}
	Existing methods for measuring training sample influence on models only provide static, overall measurements, overlooking how sample influence changes during training. We propose Dynamic Influence Tracker (DIT), which captures the time-varying sample influence across arbitrary time windows during training.
	DIT offers three key insights: 1) Samples show different time-varying influence patterns, with some samples important in the early training stage while others become important later. 2) Sample influences show a weak correlation between early and late stages, demonstrating that the model undergoes distinct learning phases with shifting priorities. 3) Analyzing influence during the convergence period provides more efficient and accurate detection of corrupted samples than full-training analysis.   Supported by theoretical guarantees without assuming loss convexity or model convergence, DIT significantly outperforms existing methods, achieving up to 0.99 correlation with ground truth and above 98\% accuracy in detecting corrupted samples in complex architectures.
\end{abstract}

\section{Introduction}
Understanding training sample influence on deep learning models remains a key challenge for model interpretability and robustness. While influence functions~\cite{koh2017understanding} and subsequent research~\cite{guo2021fastif,schioppa2022scaling,choe2024your} can measure sample influence on models, they only work after model convergence, failing to capture the sample influence during training.

This limitation motivates two fundamental questions: 

\textit{1) How does sample influence evolve during training? }

\textit{2) How can we measure and utilize these time-varying dynamics to improve model training? } 

Answering these questions is challenging. First, analyzing non-converged models requires new theoretical frameworks beyond existing influence methods that rely on model convergence ~\cite{basu2021influence}. Second, measuring time-varying influence requires intensive computation of model updates, demanding high computational resources. Third, it is required to analyze multiple aspects of model behavior, such as parameter changes, loss changes, and prediction changes, to capture the complex interplay between samples and model learning dynamics.

To address these challenges, we propose Dynamic Influence Tracker (DIT), a query-based framework that measures the time-varying influence of training samples during training. DIT first estimates how removing individual samples affects model parameters in training windows. Then, DIT computes inner products between parameter changes and task-specific query vectors to efficiently measure sample influence on model behaviors.

Using DIT, we uncover several key insights about learning dynamics.
First, samples show different time-varying influence patterns (Early Influencers, Late Bloomers, Stable Influencers, and Highly Fluctuating), revealing the limitations of traditional static influence analysis methods. 
Second, sample influences show a weak correlation between early and late stages, revealing the multi-stage nature of deep learning, where learning priorities shift significantly over time.
Third, it is more effective and efficient to identify corrupted samples by analyzing sample influence during the model convergence period than the entire training process.

In summary, DIT has the following advantages:
\begin{itemize}
	\item \textbf{Accurate and Dynamic Influence Analysis.} DIT achieves state-of-the-art accuracy with a 0.99 correlation to ground truth. As the first approach enabling influence analysis within any training window, DIT enables efficient corruption detection ($>$98\% accuracy) in single-epoch analysis.
	\item \textbf{Query-Based Multi-faceted Influence Measure.} DIT uses queries to evaluate how samples impact multiple aspects of model behavior, providing a comprehensive evaluation during training.
	\item \textbf{Robustness to Non-Convergence and Non-Convexity.} DIT theoretically guarantees influence estimation for arbitrary training windows without requiring model convergence or loss convexity, supporting real-world deep learning applications.
\end{itemize}

\section{Preliminaries}
Let $\mathcal{Z} = \mathcal{X} \times \mathcal{Y}$ denote the space of observations, where $\mathcal{X} \subseteq \mathbb{R}^d$ is the input space and $\mathcal{Y}$ is the output space.
Given a training set $D = \{z_i\}_{i=1}^N$ of i.i.d. observations $z_i = (x_i, y_i) \in \mathcal{Z}$, a model $f: \mathcal{X} \times \Theta \rightarrow \mathcal{Y}$ parameterized by $\theta \in \Theta \subseteq \mathbb{R}^p$, and a loss function $\ell: \mathcal{Z} \times \Theta \rightarrow \mathbb{R}$, we formulate the learning problem as:
\begin{equation}
	\hat{\theta} = \arg\min_{\theta \in \Theta} \frac{1}{N} \sum_{i=1}^N \ell(z_i; \theta).
\end{equation}

\begin{definition}[Stochastic Gradient Descent (SGD)]
	Let $g(z; \theta) = \nabla_\theta \ell(z; \theta)$, and SGD starts from $\theta^{[0]}$. For mini-batch $S_t \subseteq \{1, ..., N\}$ and learning rate $\eta_t$ at step $t$, SGD updates:
	\begin{equation}
		\theta^{[t+1]} = \theta^{[t]} - \frac{\eta_t}{|S_t|} \sum_{i \in S_t} g(z_i; \theta^{[t]}), \quad 0\leq t \leq T - 1,
	\end{equation}
	where $T$ is the total number of SGD steps.
\end{definition}

\begin{definition}[Influence Function~\cite{koh2017understanding}]
	The influence function measures the effect of removing sample $z_j$ on optimal parameters $\hat{\theta}$, defined as $\hat{\theta}_{-j} - \hat{\theta}$ where $\hat{\theta}_{-j} = \arg\min_\theta \sum_{i=1, i \neq j}^N \ell(z_i; \theta)$. For strongly convex losses, it can be approximated as:
	\begin{equation}
		\hat{\theta}_{-j} - \hat{\theta} \approx - \hat{H}^{-1} \nabla_\theta \ell(z_j; \hat{\theta}),
	\end{equation}
	where $\hat{H} = \frac{1}{N} \sum_{z \in D} \nabla^2 \ell(z; \hat{\theta})$ is the Hessian of the loss at the optimal parameters.
\end{definition}

\begin{definition}[Counterfactual SGD]
	The counterfactual SGD excludes the $j$-th training sample to analyze its influence. Starting from $\theta_{-j}^{[0]} = \theta^{[0]}$, the parameters are updated at each step $t$ using:
	\begin{equation}
		\theta^{[t+1]} = \theta^{[t]} - \frac{\eta_t}{|S_t|} \sum_{i \in S_t\setminus\{j\}} g(z_i; \theta_{-j}^{[t]}), \quad 0\leq t \leq T - 1.
	\end{equation}
\end{definition}

\begin{definition}[SGD-Influence~\cite{hara2019data}]
	The SGD-influence of training sample $z_j \in D$ within $t$ steps is defined as $\theta_{-j}^{[t]} - \theta^{[t]}$.
\end{definition}

While the influence function analyzes impact at the optimum, SGD-Influence measures how excluding sample $z_j$ affects the entire training process. The following sections will introduce our method for estimating sample influence during arbitrary training windows.

\section{Parameter Change in Time Window} \label{secparam}
\subsection{Problem Formulation}
Our goal is to estimate the impact of training samples during an arbitrary time window $[t_1, t_2]$, where $0\leq t_1<t_2\leq T$. We formalize this goal with a counterfactual question: how would the model's parameters change during the interval $[t_1, t_2]$ if a specific sample $z_j$ is not used?

\begin{definition}[Parameter Change in Time Window]
	For a time window $[t_1, t_2]$ during SGD training, the parameter change estimates the contribution of a training sample $z_j$ as:
	\begin{equation}
		\Delta \theta_{-j}^{[t_1,t_2]} = (\theta_{-j}^{[t_2]} - \theta_{-j}^{[t_1]}) - (\theta^{[t_2]} - \theta^{[t_1]}),
	\end{equation}
	where $(\theta^{[t_2]} - \theta^{[t_1]})$ represents the parameter changes under standard SGD within $[t_1, t_2]$, and $(\theta_{-j}^{[t_2]} - \theta_{-j}^{[t_1]})$ represents the parameter changes over the same interval when excluding sample $z_j$.
\end{definition}

For the special case $[0,t]$, starting from the beginning of training, this simplifies to:
\begin{align}
	\Delta \theta_{-j}^{[0,t]} = (\theta_{-j}^{[t]} - \theta_{-j}^{[0]}) - (\theta^{[t]} - \theta^{[0]})= \theta_{-j}^{[t]} - \theta^{[t]}.
\end{align}
For brevity, we denote $ \Delta \theta_{-j}^{[t]} = \Delta \theta_{-j}^{[0,t]} $.

\subsection{Estimation of Parameter Change in Time Window}
We aim to estimate the parameter change due to the absence of sample $z_j$ over the time window $[t_1, t_2]$, where $0\leq t_1<t_2\leq T$:
\begin{equation} \label{equdelta}
	\begin{split}
		\Delta \theta_{-j}^{[t_1,t_2]} = (\theta_{-j}^{[t_2]} - \theta_{-j}^{[t_1]}) - (\theta^{[t_2]} - \theta^{[t_1]})\\= (\theta_{-j}^{[t_2]} - \theta^{[t_2]}) - (\theta_{-j}^{[t_1]} - \theta^{[t_1]}).
	\end{split}
\end{equation}
Consider the normal SGD update for step $t$:
\begin{equation}
	\theta^{[t+1]} = \theta^{[t]} - \frac{\eta_t}{|S_t|} \sum_{i \in S_t} g(z_i; \theta^{[t]}).
\end{equation}
Consider the SGD update excluding sample $z_j$:
\begin{equation}
	\theta_{-j}^{[t+1]} = \theta_{-j}^{[t]} - \frac{\eta_t}{|S_t|} \sum_{i \in S_t\setminus\{j\}} g(z_i; \theta_{-j}^{[t]}).
\end{equation}
Calculate the difference between the two updates:
\begin{equation} \label{eqdiff}
	\begin{split}
		&\theta_{-j}^{[t+1]} - \theta^{[t+1]} =(\theta_{-j}^{[t]} - \theta^{[t]}) \\&- \frac{\eta_t}{|S_t|} (\sum_{i \in S_t\setminus\{j\}} g(z_i; \theta_{-j}^{[t]}) - \sum_{i \in S_t} g(z_i; \theta^{[t]})).
	\end{split}
\end{equation}
Approximate the gradient differences using a first-order Taylor expansion:
\begin{equation} \label{eqg}
	g(z_i; \theta_{-j}^{[t]}) - g(z_i; \theta^{[t]}) \approx \nabla_\theta g(z_i; \theta^{[t]})^T (\theta_{-j}^{[t]} - \theta^{[t]}),
\end{equation}
where $\nabla_\theta g(z_i; \theta^{[t]})$ is the gradient of $g(z_i; \theta)$ with respect to $\theta$, evaluated at $\theta^{[t]}$.
Define the approximate Hessian matrix $H^{[t]}$ as the average of the outer products of these gradients over the mini-batch:
\begin{equation}\label{eqht}
	H^{[t]} = \frac{1}{|S_t|} \sum_{i \in S_t} \nabla_\theta g(z_i; \theta^{[t]})^T.
\end{equation}

Averaging Eq.(\ref{eqg}) over $S_t$ and using Eq.(\ref{eqht}), we have:
\begin{equation} \label{eqsumst}
	\frac{1}{|S_t|} \sum_{i \in S_t} (g(z_i; \theta_{-j}^{[t]}) - g(z_i; \theta^{[t]})) \approx H^{[t]} (\theta_{-j}^{[t]} - \theta^{[t]}).
\end{equation}

Using Eq.(\ref{eqsumst}) in Eq.(\ref{eqdiff}), we have:
\begin{equation}\label{eq14}
	\begin{split}
		\theta_{-j}^{[t+1]} - \theta^{[t+1]} &\approx (I - \eta_t H^{[t]})(\theta_{-j}^{[t]} - \theta^{[t]}) \\&+ \mathbf{1}_{j \in S_t}\frac{\eta_t}{|S_t|} g(z_j; \theta^{[t]}),
	\end{split}
\end{equation}
where $\mathbf{1}_{j \in S_t}$ is an indicator function that equals 1 if $j \in S_t$, otherwise 0. The full derivation is provided in Appendix \ref{app:detailed_derivation}.

Let $Z_t = I - \eta_{t} H^{[t]}$, $ \mathbf{\tilde{1}}_j^{[t]} = \mathbf{1}_{j \in S_{t}}\frac{\eta_{t}}{|S_{t}|} g(z_j; \theta^{[t]})$ and recursively apply this relation over the interval $[t_1, t_2]$:
\begin{equation} \label{equt2}
	\begin{split}
		\theta_{-j}^{[t_2]} - \theta^{[t_2]} \approx Z_{t_2-1}Z_{t_2-2}\ldots Z_{t_1}(\theta_{-j}^{[t_1]} - \theta^{[t_1]}) \\+ \sum_{t=t_1}^{t_2-1} Z_{t_2-1}Z_{t_2-2}\ldots Z_{t+1} \mathbf{\tilde{1}}_j^{[t]}.
	\end{split}
\end{equation}
Combining Eq. (\ref{equdelta}) and Eq. (\ref{equt2}), we can get:
\begin{equation} \label{equdelta1}
	\begin{split}
		\Delta \theta_{-j}^{[t_1,t_2]} \approx \left(\prod_{k=t_1}^{t_2-1} Z_k - I\right)(\theta_{-j}^{[t_1]} - \theta^{[t_1]}) \\+ \sum_{t=t_1}^{t_2-1} \left(\prod_{k=t+1}^{t_2-1} Z_k\right) \mathbf{\tilde{1}}_j^{[t]}.
	\end{split}
\end{equation}

We use Eq. (\ref{equdelta1}) for the interval $[0, t_1]$ with $	\theta_{-j}^{[0]}= \theta^{[0]} $ to get $(\theta_{-j}^{[t_1]} - \theta^{[t_1]})$:
\begin{equation}\label{eqt1jt1}
	\Delta \theta_{-j}^{[0,t_1]}=	\theta_{-j}^{[t_1]} - \theta^{[t_1]} \approx \sum_{t=0}^{t_1-1} \left(\prod_{k=t+1}^{t_1-1} Z_k\right) \mathbf{\tilde{1}}_j^{[t]}.
\end{equation}
Substituting Eq. (\ref{eqt1jt1}) into Eq. (\ref{equdelta1}), we obtain:
\begin{equation} \label{eqestimate}
	\begin{split}
		\Delta \theta_{-j}^{[t_1,t_2]} \approx & \left(\prod_{k=t_1}^{t_2-1} Z_k - I\right)\left(\sum_{t=0}^{t_1-1} \left(\prod_{k=t+1}^{t_1-1} Z_k\right) \mathbf{\tilde{1}}_j^{[t]}\right) \\ 
		&+ \sum_{t=t_1}^{t_2-1} \left(\prod_{k=t+1}^{t_2-1} Z_k\right) \mathbf{\tilde{1}}_j^{[t]}.
	\end{split}
\end{equation}
We denote the right-hand side of Eq. (\ref{eqestimate}) as $\widehat{\Delta \theta}_{-j}^{[t_1,t_2]}$, which represents our estimator of the parameter change. 

\subsection{Estimation Error Analysis} 
To analyze the quality of this estimator, we derive theoretical bounds on the estimation error $\|\Delta \theta_{-j}^{[t_1,t_2]} - \widehat{\Delta \theta}_{-j}^{[t_1,t_2]}\|$. 
Our analysis provides theoretical guarantees in non-convex settings without requiring model convergence, with bounds that hold for arbitrary training intervals. This makes DIT particularly suitable for real-world deep learning scenarios. The complete analysis can be found in Appendix \ref{app:estimation_error_proof}.

\section{Dynamic Influence Tracker}
While Section \ref{secparam} establishes how to measure sample influence on model parameters, these parameter changes may not directly reflect their impact on different model behaviors.

\subsection{Query-Based DIT} \label{secqueryDIT}
To understand the impact on model behaviors (e.g., test loss, predictions), we propose Dynamic Influence Tracker (DIT) to project parameter changes onto meaningful directions using query vectors, where each query vector defines a direction of interested model behavior.

\begin{definition}[Query-based Dynamic Influence Tracker]
	Let $q: [0, T] \rightarrow \mathbb{R}^p$ be a query function that maps time $t$ to a query vector $q(t) \in \mathbb{R}^p$. The Query-based Dynamic Influence Tracker for a training sample $z_j$ over the time window $[t_1, t_2]$ is defined as:
	\begin{equation}
		Q_{-j}^{[t_1,t_2]}(q) = \langle q(t_2), \Delta\theta_{-j}^{[t_2]} \rangle - \langle q(t_1), \Delta\theta_{-j}^{[t_1]} \rangle, 
	\end{equation}
	where $\Delta\theta_{-j}^{[t]} = \theta_{-j}^{[t]} - \theta^{[t]}$ represents the parameter change at time $t$ and $\langle \cdot, \cdot \rangle$ denotes the standard inner product in $\mathbb{R}^p$. 
\end{definition}

Using the test loss gradient as query vector, $q(t) = \nabla_\theta \ell(z_{\text{test}}; \theta^{[t]})$, we have:
\begin{equation}
	\begin{split}
		&Q_{-j}^{[t_1,t_2]}(q) 
		\\&= \langle \nabla_\theta \ell(z_{\text{test}}; \theta^{[t_2]}), \Delta\theta_{-j}^{[t_2]} \rangle - \langle \nabla_\theta \ell(z_{\text{test}}; \theta^{[t_1]}), \Delta\theta_{-j}^{[t_1]} \rangle 
		\\&\approx [\ell(z_{\text{test}}; \theta_{-j}^{[t_2]}) - \ell(z_{\text{test}}; \theta_{-j}^{[t_1]})]\\ & \quad\quad\quad\quad\quad\quad - [\ell(z_{\text{test}}; \theta^{[t_2]}) - \ell(z_{\text{test}}; \theta^{[t_1]})]. 
	\end{split}
\end{equation}

Different choices of $q$ enable analysis of various model characteristics. We can set $q = \nabla_\theta f(x_\text{test}; \theta^{[t]})$ measures prediction changes, $q = e_i$ (standard basis vector) examines individual parameter importance, and $q = \nabla_\theta \ell(z_j; \theta^{[t]})$ assesses gradient alignments. A detailed analysis of these query vectors is in Appendix \ref{app:dit_approximation}.

\subsection{DIT Implementation}
To compute $Q_{-j}^{[t_1,t_2]}(q)$, we propose Algorithm \ref{alg:dit_training} for collecting essential information during model training to estimate parameter changes, and Algorithm \ref{alg:dit_inference} for efficient influence computation using these values.

\subsubsection{Model Training}
Algorithm \ref{alg:dit_training} presents the model training process with information collection. Users can set $W$ to cover potential query ranges while avoiding full trajectory storage.

\begin{algorithm}[H]
	\caption{Model Training} \label{alg:dit_training}
	\begin{algorithmic}[1]
		\Require Training dataset $D = \{z_n\}_{n=1}^N$, learning rate $\eta_t$, batch size $|S_t|$, training steps $T$, selectable storage window $W$
		\Ensure Stored information $A$ 
		\State Initialize model parameters $\theta^{[0]}$
		\State Initialize an empty sequence $A$
		\For {$t = 0$ to $T-1$}
		\State $S_t = \text{SampleBatch}(D, |S_t|)$
		\State $\theta^{[t+1]} = \theta^{[t]} - \frac{\eta_t }{|S_t|} \sum_{i \in S_t} g(z_i; \theta^{[t]})$
		\State \textbf{if} $t \in W$ \textbf{then} $A[t] = \{S_t, \eta_t, \theta^{[t+1]} \}$
		\EndFor
		\State \textbf{return} $A$
	\end{algorithmic}
\end{algorithm}

To estimate parameter changes, we collect minimal but sufficient information within user-specified time windows $W$ during standard SGD training, avoiding the storage of full training trajectories required by traditional methods.

For scenarios with strict storage constraints, we design a checkpoint-based implementation (Appendix \ref{app:dit_Checkpoints}) that greatly reduces storage to $O(Ep)$ while maintaining accuracy.  

\subsubsection{Influence Computation}
Given the trajectory information, Algorithm \ref{alg:dit_inference} computes $Q_{-j}^{[t_1,t_2]}(q)$ through efficient backward propagation.

\begin{algorithm}[H]
	\caption{DIT Sample Influence Computation}\label{alg:dit_inference}
	\begin{algorithmic}[1]
		\Require Stored information $A $, query function $q$, time window $[t_1, t_2] $, specified sample $z_j$
		\Ensure Estimated influence $Q $ for sample $z_j$
		\State Initialize $Q \gets 0 $, $u_1^{[t_2 - 1]} \gets 0 $
		\State Initialize $u_2^{[t_2 - 1]} \gets q(t_2) $
		\For {$t = t_2 - 1$ \textbf{downto} $0$}
		\If{$j \in S_t $}
		\State $Q \gets Q + \left\langle (u_2^{[t]} - u_1^{[t]}), \dfrac{\eta_t}{|S_t|} g(z_j; \theta^{[t]}) \right\rangle $
		\EndIf
		\State $u_1^{[t-1]} \gets u_1^{[t]} - \eta_t H^{[t]} u_1^{[t]} $
		\State $u_2^{[t-1]} \gets u_2^{[t]} - \eta_t H^{[t]} u_2^{[t]} $
		\State \textbf{if} $t = t_1$ \textbf{then} \; $u_1^{[t-1]} \gets q(t_1)$ \allowbreak
		\EndFor
		\State \textbf{return} $Q$
	\end{algorithmic}
\end{algorithm}
Algorithm~\ref{alg:dit_inference} utilizes two key variables, $u_2^{[t]} $ and $u_1^{[t]} $, which propagate $q(t_2) $ and $q(t_1)$ backwards through time while incorporating the $Z_k$ matrices. The algorithm computes $Q$ by summing the inner products of $(u_2^{[t]} - u_1^{[t]}) $ with $\mathbf{\tilde{1}}_j^{[t]} $ at each time step. The final sum $Q$ matches $Q_{-j}^{[t_1,t_2]}(q)$ as defined in Section \ref{secqueryDIT}. See Appendix \ref{appenproofalg2} for proof.

DIT avoids the computationally intensive direct computation and storage of the Hessian matrices, which typically requires $O(Tp^2)$ operations. Instead, DIT efficiently computes Hessian-vector products $H^{[t]}u = \nabla_\theta \langle u, g(z; \theta^{[t]}) \rangle$, requiring only $O(|S_t|p)$ operations per iteration. This optimization effectively handles large models and datasets in modern machine-learning contexts.

\section{Experiments}
We evaluate DIT by examining how training sample influences evolve (Section \ref{secpsid}), comparing DIT's influence estimation accuracy against existing methods (Sections \ref{secinfluacc}, \ref{secpattern}), analyzing influence patterns across training stages (Section \ref{secinflusimi}), and showing DIT time window analysis benefits practical ML applications (Section \ref{secappdit}).

\subsection{Experimental Setup}
Experiments were conducted on eight NVIDIA RTX A5000 GPUs (24GB each), dual Intel Xeon Gold 6342 CPUs (2.80GHz, 96 cores), and 503GB RAM. Implementation uses Ubuntu 22.04.3 LTS, PyTorch v2.4.1, CUDA 12.4, and Python 3.11.9. All results are reported as mean ± standard deviation over 16 runs with different random seeds. 

\paragraph{Datasets}
We evaluate on ImageNet-1K~\cite{deng2009imagenet}, Adult~\cite{Dua2019}, 20Newsgroups~\cite{Lang95}, and MNIST/EMNIST~\cite{lecun-mnisthandwrittendigit-2010,Cohen2017EMNIST}, covering image, tabular, and text domains.  

\paragraph{Models} 
We used 1) Logistic Regression (LR) as a convex baseline, 2) Deep Neural Network (DNN), 3) Convolutional Neural Network (CNN), and 4) a Vision Transformer with 16x16 patch size (ViT-B/16). The DNN, CNN, and ViT-B/16 represent non-convex scenarios.  

Detailed specifications are provided in Appendix \ref{appexp}.

\paragraph{Comparison Methods} 
We compare DIT against two baseline methods:

1) \textbf{Leave-One-Out (LOO)} serves as ground truth, measuring influence by retraining without sample $z_j$.
$ \Delta \ell_{LOO}(z_j) = \frac{1}{M} \sum_{i=1}^M \left( \ell(z_i, \theta_{-j}) - \ell(z_i, \theta) \right)$, where $z_i \in D_{\text{test}}$, $M$ is the size of the test set $D_{\text{test}} = \{ z_i \}_{i=1}^{M}$.  

2) \textbf{Influence Functions (IF)}~\cite{koh2017understanding} estimates the influence of removing a training sample $z_j$ on the model's overall loss for a test set $D_{\text{test}}$: 
$ I(z_j, D_{\text{test}}) = -\frac{1}{M} \sum_{i=1}^M \nabla_\theta \ell(z_i, \theta)^T H^{-1} \nabla_\theta \ell(z_j, \theta) $,
where $H$ is the Hessian of the model's loss at $\theta$.

3) \textbf{DIT} estimates influence by setting $q(t) = \frac{1}{M} \sum_{i=1}^M \nabla_\theta \ell(z_i; \theta^{[t]}) $, measuring the impact on test set $D_{\text{test}}$ loss across time window $[t_1, t_2]$: 
$Q_{-j}^{[t_1,t_2]}(q) \approx \frac{1}{M} \sum_{i=1}^M \left[ \ell(z_i; \theta_{-j}^{[t_2]}) - \ell(z_i; \theta_{-j}^{[t_1]}) \right] - \frac{1}{M} \sum_{i=1}^M \left[ \ell(z_i; \theta^{[t_2]}) - \ell(z_i; \theta^{[t_1]}) \right]$.

\subsection{Patterns of Sample Influence Dynamics} \label{secpsid}
While existing methods provide a static estimate of sample influence for the entire training process, our study reveals that influence on model performance evolves dynamically during training. To uncover this, we conducted a preliminary exploration using LOO, as it provides a ground truth assessment of each sample's influence on model performance. Our methodology involved randomly selecting 256 training samples and using LOO to evaluate their loss change at each epoch during model training.  For each sample at each epoch, we temporarily removed it from training, retrained the model for that epoch, and recorded the resulting change in loss. This process generated a time series of sample influences, allowing us to track how the importance of each sample evolved throughout the training process.

As the model converges during training, the loss change decreases with increasing epochs. To identify patterns of sample influence rather than relative influences, we normalized the values within each epoch using StandardScaler. We then used linear regression to analyze trends in influence changes. Figure \ref{fig:influence_dynamics} shows four distinct influence evolution patterns, displaying centroid values for each group. Detailed experimental settings are provided in Appendix \ref{apppattern}. 
\begin{figure}[t]
	\centering
	\includegraphics[width=0.99\linewidth]{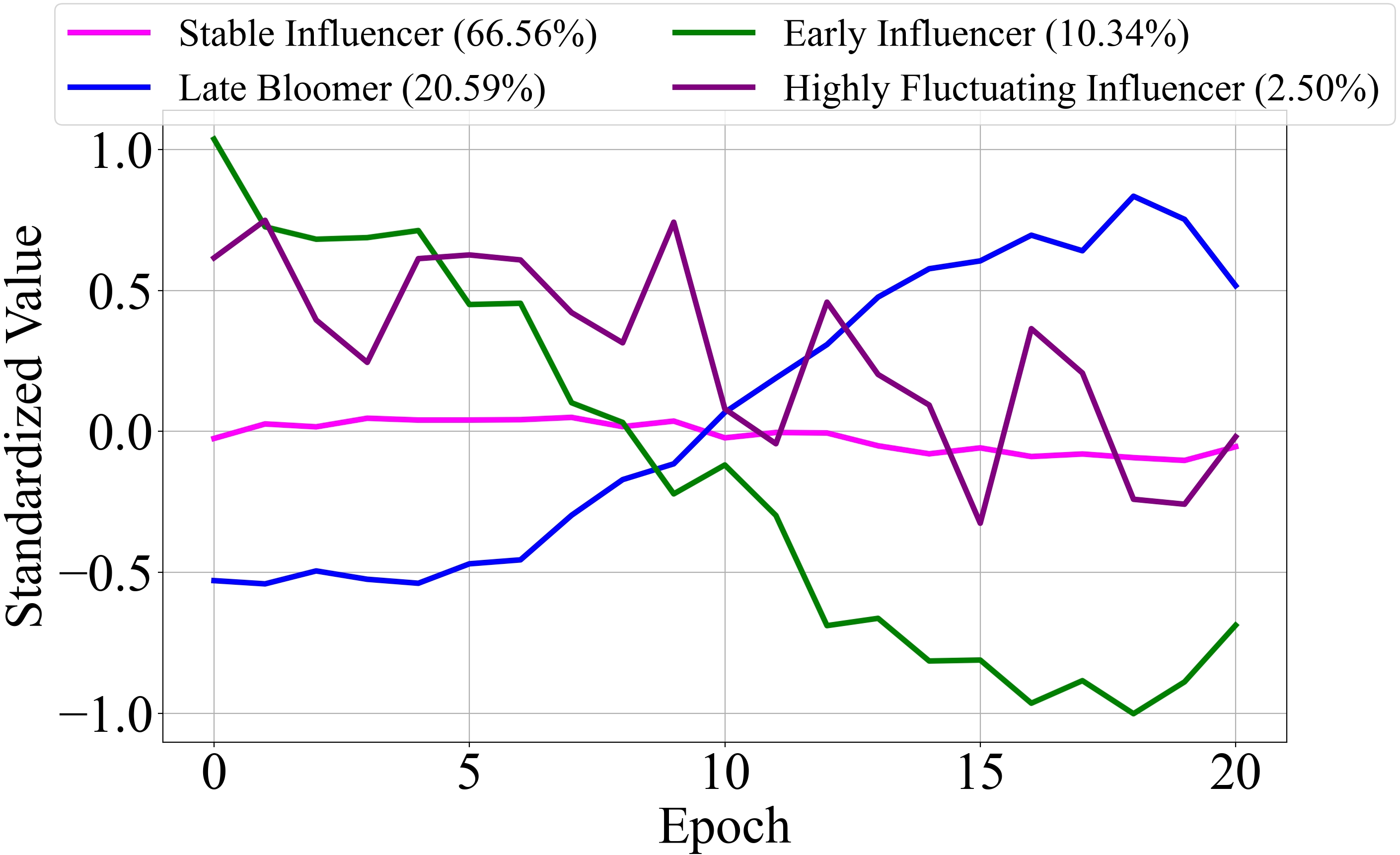}
	\caption{ Illustration of influence dynamics patterns for MNIST training using DNN}
	\label{fig:influence_dynamics}
\end{figure}
\begin{itemize}
	\item \textbf{Early Influencers}: Significant early impact that diminishes over time.
	\item \textbf{Late Bloomers}: Gain influence as training progresses.
	\item \textbf{Stable Influencers}: Consistent influence throughout training.
	\item \textbf{Highly Fluctuating Influencers}: Large variations in influence across training.
\end{itemize}
We further analyzed the pattern distribution across datasets and models, as shown in Table \ref{tab:cluster_distribution}. 

\begin{table}[h]
	\small
	\centering
	\caption{Distribution of influence dynamic patterns across datasets and models (percentage)}
	\label{tab:cluster_distribution}
	\setlength{\tabcolsep}{0.5pt}  
	\begin{tabular}{c|c|cccc}
		\toprule
		Model & Dataset & \begin{tabular}[c]{@{}c@{}}\textbf{Stable} \\ \textbf{Influencer}\end{tabular} & \begin{tabular}[c]{@{}c@{}}\textbf{Early} \\ \textbf{Influencers}\end{tabular} & \begin{tabular}[c]{@{}c@{}}\textbf{Late} \\ \textbf{Bloomers}\end{tabular} & \begin{tabular}[c]{@{}c@{}}\textbf{Highly} \\ \textbf{Fluctuating}\end{tabular} \\
		\cmidrule(lr){1-6}
		\multirow{4}{*}{LR} 
		& Adult & 64.75±7.20 & 11.67±3.27 & 20.15±5.87 & 3.42±1.82 \\
		& 20News & 85.94±5.38 & 1.17±1.28 & 5.57±1.26 & 7.32±4.24 \\ 
		& MNIST & 80.16±12.10 & 0.79±0.96 & 10.78±9.35 & 8.27±3.36 \\
		& EMNIST & 75.49±8.40 & 0.70±0.53 & 13.77±6.77 & 10.04±2.75 \\
		\midrule
		\multirow{4}{*}{DNN} 
		& Adult & 97.91±2.66 & 0.313±1.12 & 1.00±1.55 & 0.78±0.89 \\
		& 20News & 79.03±7.78 & 8.44±4.11 & 11.41±3.90 & 1.13±0.83 \\ 
		& MNIST & 66.56±13.26 & 10.34±4.65 & 20.59±9.44 & 2.50±0.93 \\
		& EMNIST & 78.16±14.48 & 7.09±7.678 & 7.47±9.87 & 7.28±3.55 \\
		\midrule
		\multirow{2}{*}{CNN} 
		& MNIST & 83.76±19.91 & 0.34±0.42 & 11.74±16.60 & 4.15±3.94 \\
		& EMNIST & 86.50±7.50 & 1.87±5.15 & 1.59±3.91 & 10.03±2.48 \\
		\bottomrule
	\end{tabular}
\end{table}

These results show several key insights.
1) All datasets and models show diverse influence patterns, with Stable Influencers dominating but other patterns consistently present. This underscores the dynamic nature of sample influence throughout the training process.
2) The presence of Early Influencers and Late Bloomers highlights the importance of time-varying analysis in understanding sample influence. DIT's ability to capture these time-varying dynamics provides a significant advantage over static influence estimation methods.
3) The varying pattern distributions across model-dataset combinations show a complex interplay between data characteristics and model architecture, emphasizing the necessity of dynamic influence analysis approaches.

\subsection{Influence Estimation Accuracy}  \label{secinfluacc}
To validate DIT's accuracy in estimating influence, we compared DIT against IF using LOO as ground truth. We employed DIT's full-time window $[0, T]$ for a fair comparison with IF, which can only measure overall sample influence on the final model. To evaluate how closely DIT and IF approximate LOO, we adopted four metrics: Pearson and Spearman correlations for linear and monotonic relationships, respectively, Kendall's tau for ordinal relationships, and Jaccard similarity for the top 30\% influencers. Detailed metric descriptions are in Appendix \ref{appmetrics}.

\begin{table*}[h]
	\small
	\centering
	\caption{Performance comparison of DIT and IF for Logistic Regression and Deep Neural Network}
	\label{tab:precise_performance}
	\setlength{\tabcolsep}{1.5pt}
	\begin{tabular}{c|c|cc|cc|cc|cc}
		\toprule
		Model & Dataset & \multicolumn{2}{c|}{Pearson} & \multicolumn{2}{c|}{Spearman} & \multicolumn{2}{c|}{Kendall's Tau} & \multicolumn{2}{c}{Jaccard} \\
		\cmidrule(lr){3-10}
		& & DIT & IF & DIT & IF & DIT & IF & DIT & IF \\
		\midrule
		\multirow{3}{*}{LR} 
		& Adult & \textbf{0.99±0.01} & 0.91±0.04& \textbf{0.99±0.01} & 0.93±0.02 & \textbf{0.95±0.01} & 0.79±0.04 & \textbf{0.91±0.04} & 0.71±0.06 \\
		& 20News & \textbf{0.99±0.01} & 0.90±0.13& \textbf{0.99±0.01} & 0.94±0.08 & \textbf{0.97±0.01} & 0.84±0.13 & \textbf{0.95±0.03} & 0.78±0.16 \\
		& MNIST & \textbf{0.93±0.10} & 0.76±0.14 & \textbf{0.98±0.01} & 0.61±0.22 & \textbf{0.95±0.02} & 0.49±0.21 & \textbf{0.91±0.05} & 0.48±0.14 \\
		\midrule
		\multirow{3}{*}{DNN}
		& Adult & \textbf{0.95±0.02} & 0.88±0.04 & \textbf{0.95±0.03} & 0.86±0.04 & \textbf{0.83±0.06} & 0.69±0.05 & \textbf{0.75±0.08} & 0.56±0.07 \\
		& 20News & \textbf{0.85±0.07} & 0.77±0.05 & \textbf{0.85±0.08} & 0.80±0.06 & \textbf{0.71±0.08} & 0.62±0.07 & \textbf{0.67±0.08} & 0.55±0.07\\
		& MNIST & \textbf{0.90±0.07} & 0.25±0.28 & \textbf{0.98±0.01} & 0.26±0.33 & \textbf{0.90±0.03} & 0.19±0.24 & \textbf{0.85±0.05} & 0.27±0.19 \\
		\bottomrule
	\end{tabular}
\end{table*}

\begin{table*}[t!]
	\small
	\centering
	\caption{Pattern-specific performance comparison between DIT and IF using MNIST-DNN.}
	\label{tab:precise_performance_pattern}
	\setlength{\tabcolsep}{1.5pt}
	\begin{tabular}{c|cc|cc|cc|cc}
		\toprule
		Sample Pattern& \multicolumn{2}{c|}{Pearson} & \multicolumn{2}{c|}{Spearman} & \multicolumn{2}{c|}{Kendall's Tau} & \multicolumn{2}{c}{Jaccard} \\
		\cmidrule(lr){2-9}
		& DIT & IF & DIT & IF & DIT & IF & DIT & IF \\
		\midrule
		Stable Influencers & \textbf{0.95±0.03} & 0.23±0.39& \textbf{0.96±0.03} & 0.16±0.40& \textbf{0.87±0.05} & 0.13±0.28& \textbf{0.82±0.12} & 0.26±0.16\\
		Early Influencers & \textbf{0.94±0.04} & 0.35±0.29& \textbf{0.98±0.01} & 0.35±0.30 & \textbf{0.92±0.03} & 0.26±0.23 & \textbf{0.89±0.07} & 0.29±0.20 \\
		Late Bloomers & \textbf{0.98±0.02} & 0.23±0.46& \textbf{0.98±0.02} & 0.19±0.38 & \textbf{0.90±0.05} & 0.15±0.27 & \textbf{0.85±0.10} & 0.27±0.21 \\
		Highly Fluctuating & \textbf{0.76±0.18} & -0.10±0.54 & \textbf{0.72±0.18} & -0.08±0.48 & \textbf{0.63±0.21} & -0.09±0.40 & \textbf{0.52±0.34} & 0.15±0.17\\
		\bottomrule
	\end{tabular}
\end{table*}

\begin{figure}[H] 
	\centering
	\subfigure[LR, adult]{
		\includegraphics[width=0.45\linewidth]{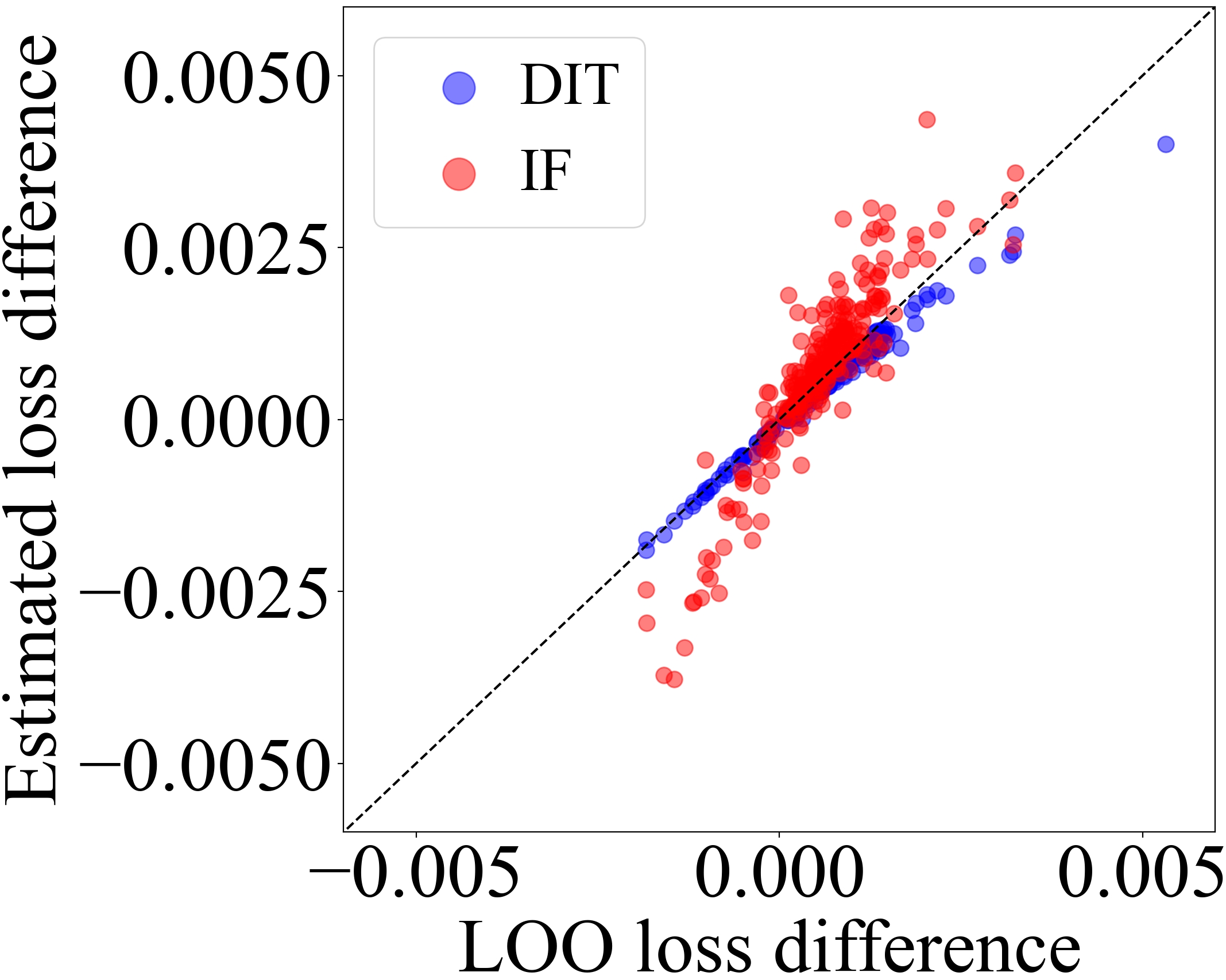} 
		\label{fig:image1}
	}
	\subfigure[DNN, adult]{
		\includegraphics[width=0.45\linewidth]{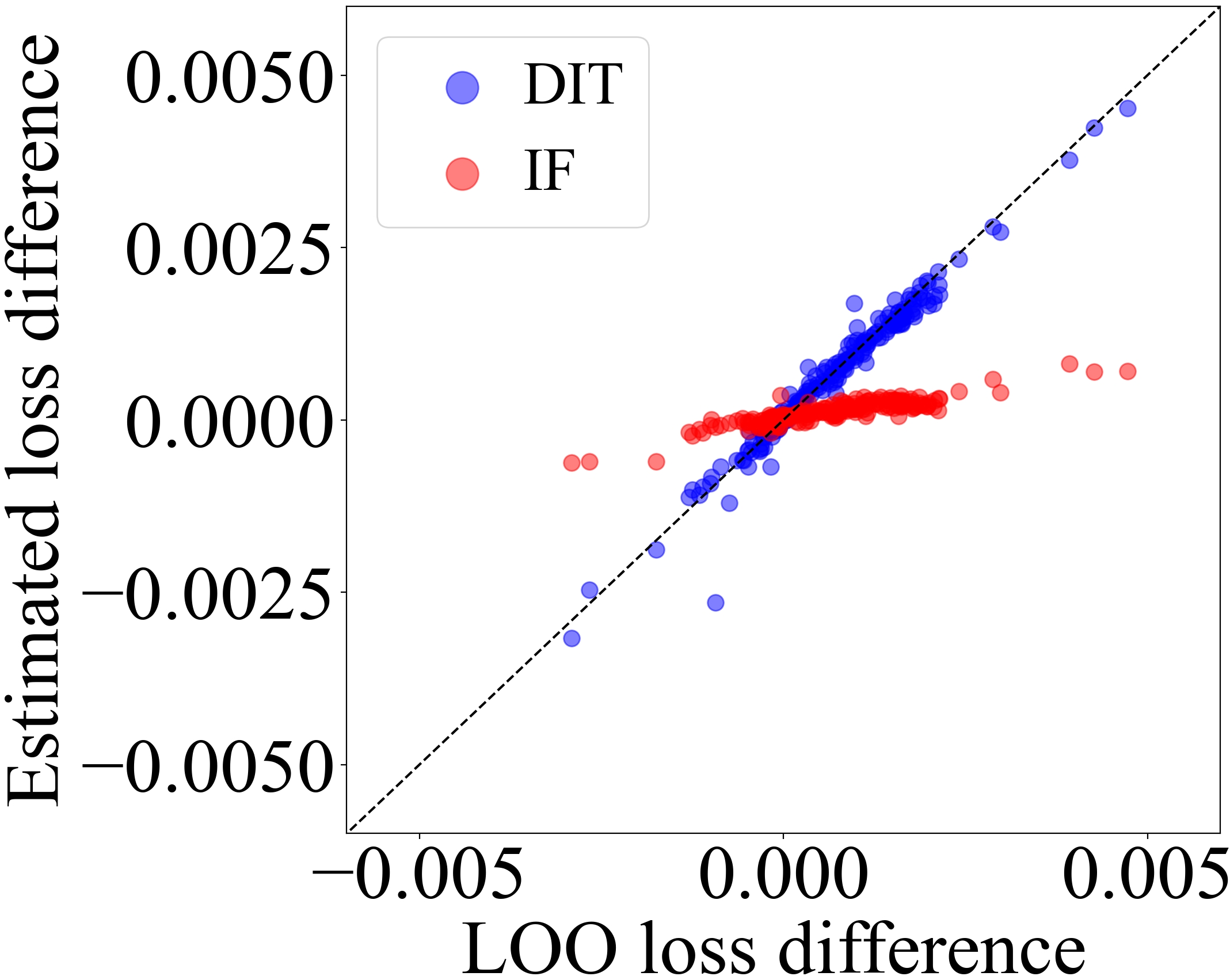} 
		\label{fig:image4}
	}
	
	\subfigure[LR, 20news]{
		\includegraphics[width=0.45\linewidth]{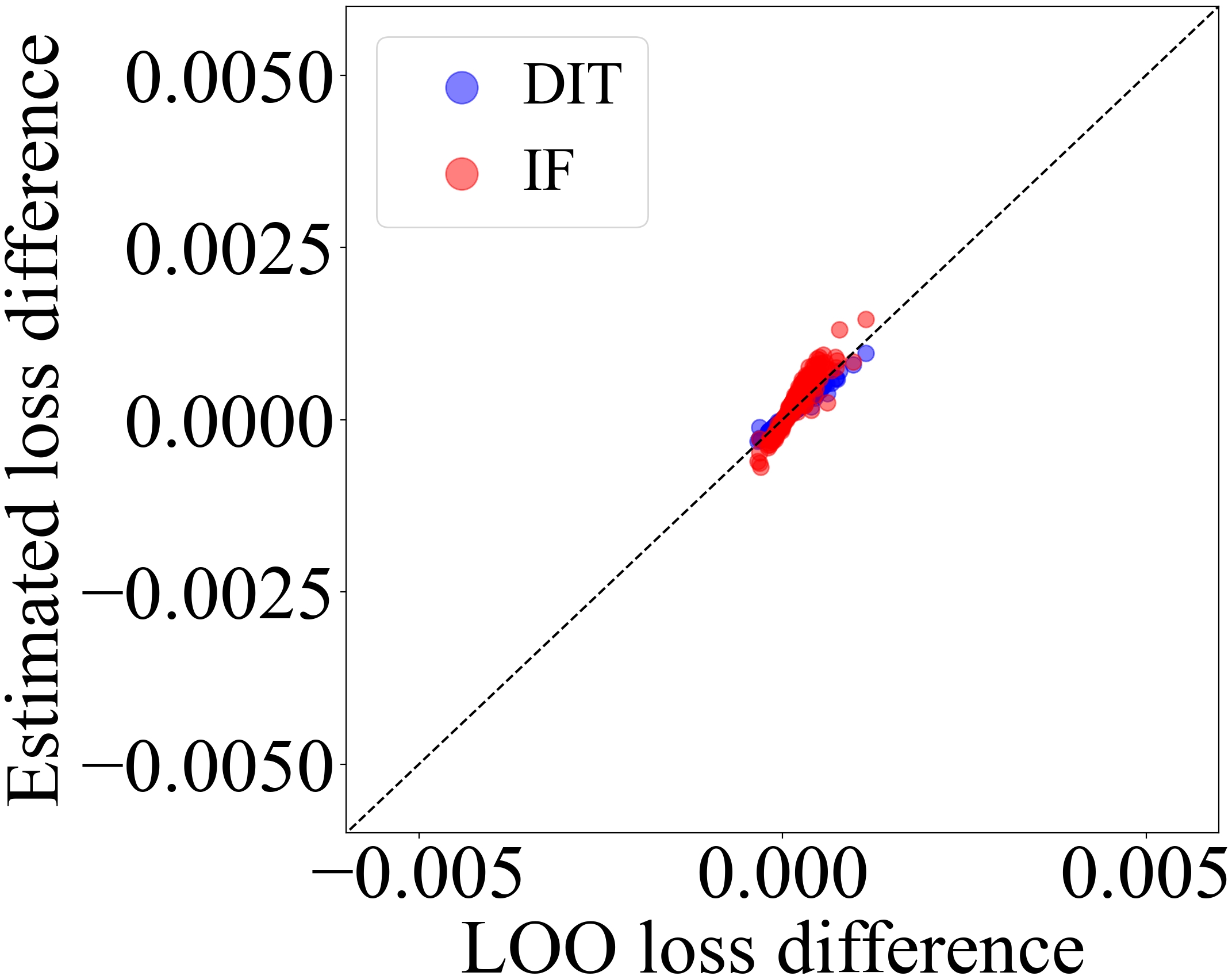}
		\label{fig:image2}
	}
	\subfigure[DNN, 20news]{
		\includegraphics[width=0.45\linewidth]{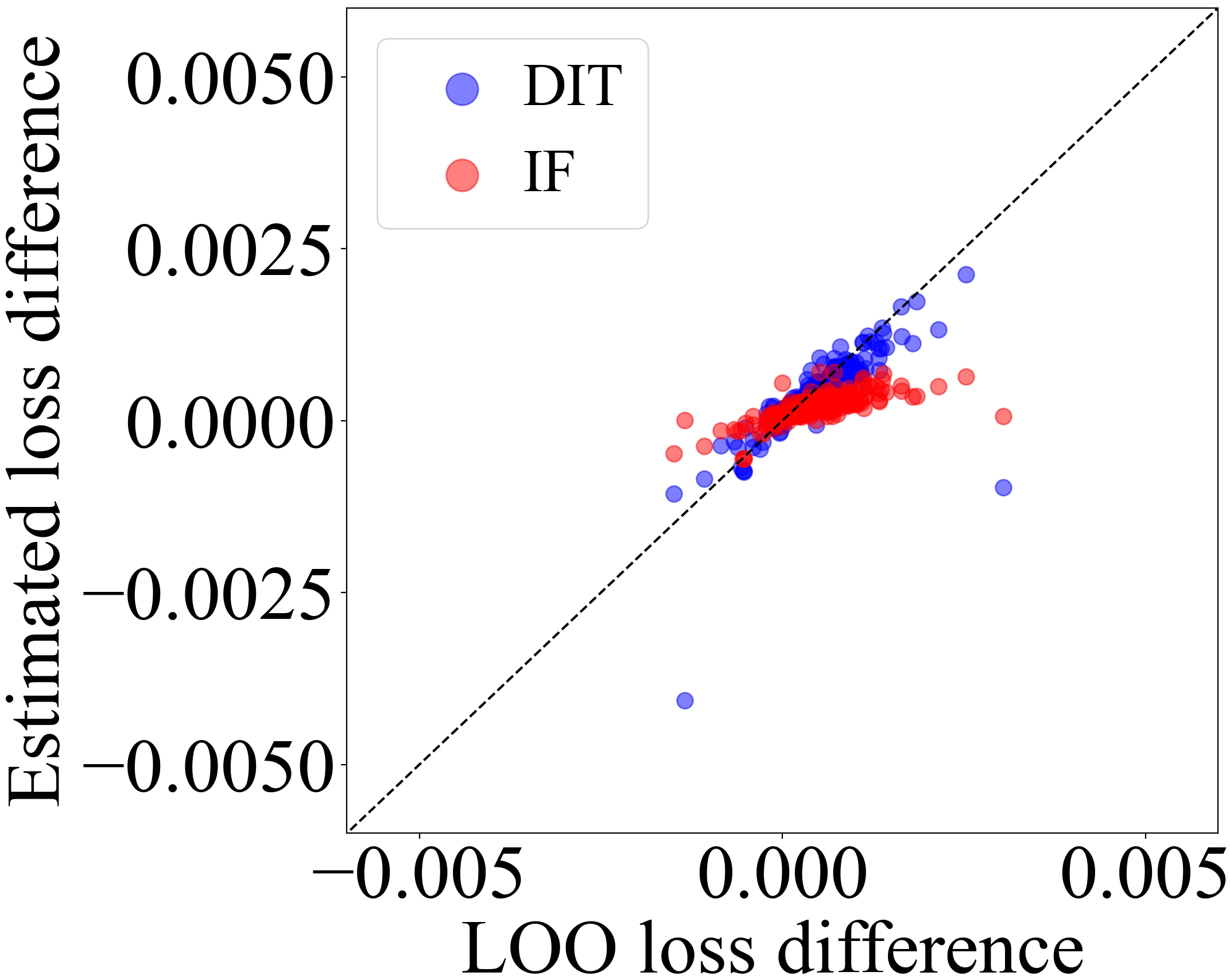}
		\label{fig:image5}
	}
	
	\subfigure[LR, mnist]{
		\includegraphics[width=0.45\linewidth]{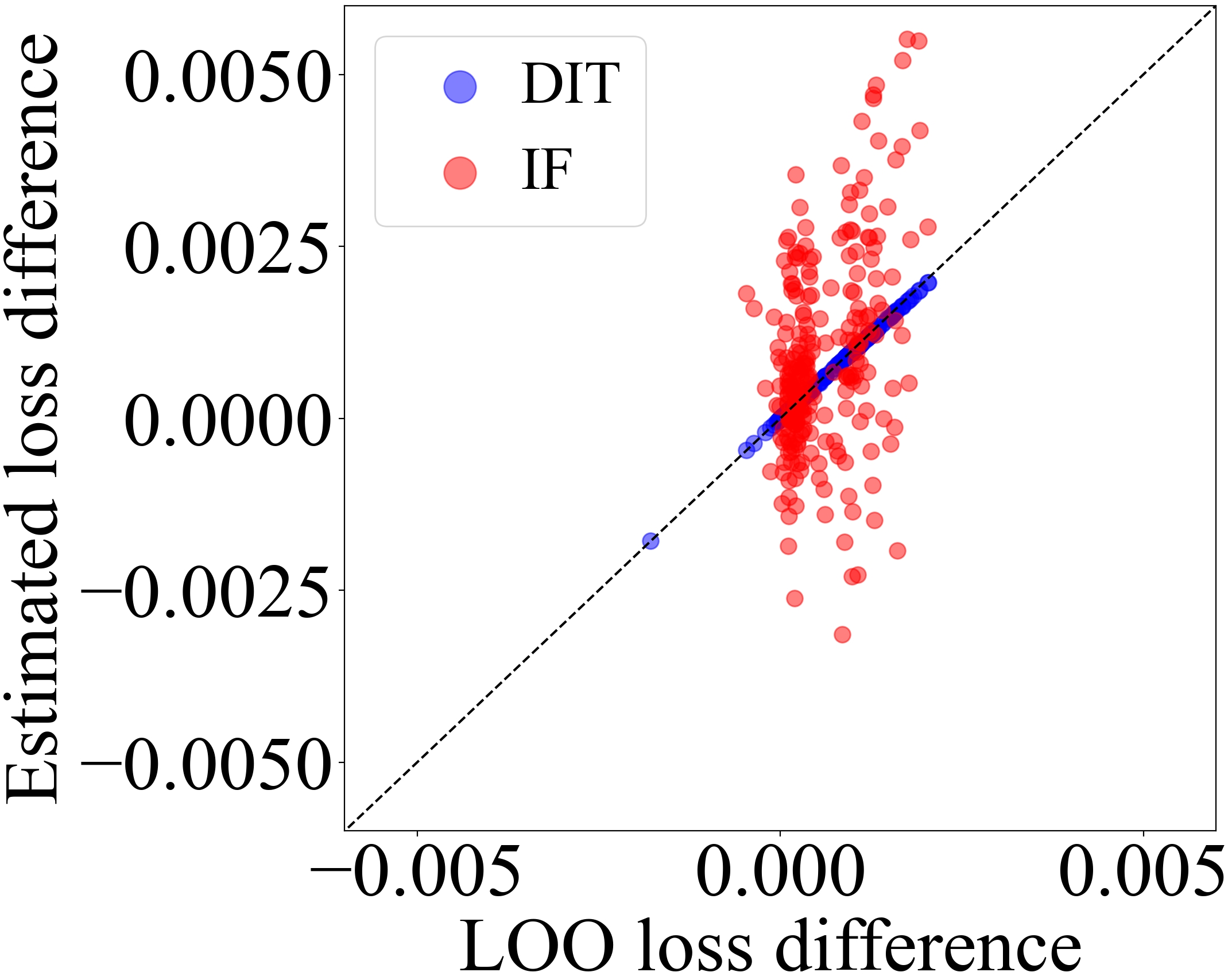}
		\label{fig:image3}
	}
	\subfigure[DNN, MNIST]{
		\includegraphics[width=0.45\linewidth]{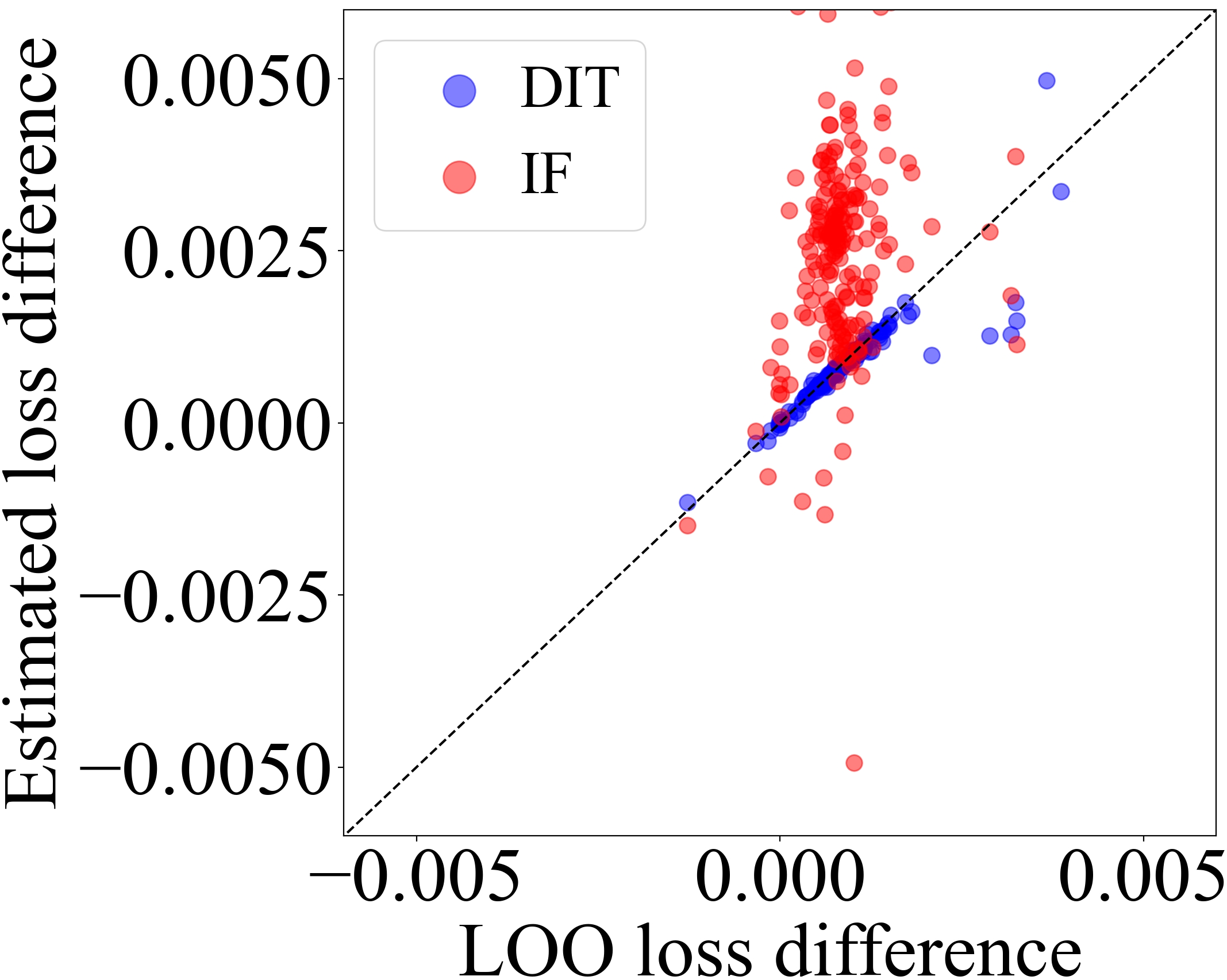}
		\label{fig:image6}
	}
	\caption{Comparison of influence estimates for DIT and IF vs. LOO ground truth across datasets using LR and DNN. The x-axis represents the ground truth influence values obtained from the LOO method. The y-axis shows DIT (blue) and IF (red) estimates.} 
	\label{fig:scatters}
\end{figure}

Table \ref{tab:precise_performance} shows several key findings: 
First, DIT consistently surpasses IF in accuracy across all datasets and model architectures, achieving correlations of up to 0.99 (Pearson and Spearman) for LR models. 
Second, DIT's advantage is most significant in complex settings like non-convex DNN and MNIST, where it maintains high correlations (0.90-0.98) while IF's performance drops significantly (0.25-0.33). 
Third, DIT shows superior robustness and reliability, with lower standard deviations (typically ±0.01-0.08) across runs compared to IF (up to ±0.33).

These results are visually shown in Figure \ref{fig:scatters}. DIT estimates closely align with the $y=x$ line, indicating superior accuracy to IF, especially with non-convex models and complex datasets.

\subsection{Pattern-Specific Influence Estimation} \label{secpattern} 
To further evaluate DIT's effectiveness across different patterns, we conducted a pattern-specific performance analysis comparing DIT against IF with LOO as ground truth using the MNIST dataset with DNN architecture. Table \ref{tab:precise_performance_pattern} presents the comparative results across multiple metrics. 

The pattern-specific analysis reveals three key findings. First, DIT significantly outperforms IF across all patterns. Second, DIT shows remarkable robustness by maintaining strong positive correlations for highly fluctuating samples where IF shows negative correlations. Third, DIT exhibits consistent performance across all correlation metrics with minimal variations, underlining its reliability for practical applications.

\begin{table*}[t]
	\small
	\centering
	\caption{ Kendall's Tau correlations across training stages}
	\label{tab:influence_stability}
	\setlength{\tabcolsep}{1.5pt}
	\begin{tabular}{llcccccc}
		\hline
		Model & Dataset & Early-Middle & Early-Late & Middle-Late & Early-Full & Middle-Full & Late-Full \\
		\hline
		\multirow{3}{*}{LR} 
		& Adult & 0.64 ± 0.14& 0.62 ± 0.08 & 0.79 ± 0.14 & 0.81 ± 0.05 & \textbf{0.82 ± 0.12}& 0.79± 0.05\\
		& 20News & 0.79 ± 0.12& 0.78 ± 0.10 & 0.79 ± 0.09 &\textbf{ 0.91 ± 0.02} & 0.88± 0.10&0.86± 0.12\\
		& MNIST &0.43 ± 0.14 & 0.15 ± 0.12 & 0.35 ± 0.14 & 0.71 ± 0.08 & \textbf{0.72± 0.09 }& 0.30 ± 0.14 \\
		& EMNIST &0.73 ± 0.04& 0.40 ± 0.16 & 0.51 ± 0.18 & 0.83 ± 0.03 & \textbf{0.89± 0.02} & 0.49 ± 0.17 \\
		\hline
		\multirow{3}{*}{DNN} 
		& Adult & 0.61 ± 0.11 & 0.41 ± 0.15 & 0.70 ± 0.06 & 0.7 ± 0.09 & \textbf{0.87 ± 0.04} & 0.69 ± 0.08 \\
		& 20news & 0.66 ± 0.06 & 0.57± 0.07 & 0.76 ± 0.05 & 0.81 ± 0.03 & \textbf{0.82 ± 0.04 }& 0.76 ± 0.04 \\
		& MNIST & 0.56 ± 0.06 & 0.18 ± 0.21 & 0.20 ± 0.25 & 0.74 ± 0.03 & \textbf{0.81± 0.04 }& 0.20 ± 0.25\\
		& EMNIST & 0.60 ± 0.12 & 0.40 ± 0.20 & 0.59 ± 0.21& 0.69 ± 0.11 & \textbf{0.84± 0.07 }& 0.63± 0.17\\
		\hline
	\end{tabular}
\end{table*}

\begin{table*}[t!]
	\small
	\centering
	\caption{Number of correctly identified flipped samples }
	\label{tab:flipmnist}
	\setlength{\tabcolsep}{2.5pt}
	\begin{tabular}{cccccccc}
		\hline
		Flipped & Model  & IF &  Full DIT &  LOO  &  \begin{tabular}[c]{@{}c@{}} DIT  \\  First Epoch \end{tabular}&  \begin{tabular}[c]{@{}c@{}}DIT\\  Mid Epoch \end{tabular}&  \begin{tabular}[c]{@{}c@{}} DIT\\   Last Epoch\end{tabular}\\
		\hline
		\multirow{3}{*}{5\%} & LR & 10.50 ± 0.50 & \textbf{10.94 ± 0.90} & \textbf{10.94 ± 0.90} & 10.56 ± 1.22 & 10.88 ± 0.78 & 10.88 ± 0.78 \\
		& DNN & 2.94 ± 2.01 & 9.06 ± 1.85 & 8.81 ± 1.98 & 8.25 ± 2.33 & 8.88 ± 2.09 & \textbf{9.38 ± 1.98} \\
		& CNN & 5.88 ± 2.26 & 10.50 ± 1.32 & 10.44 ± 1.32 & 8.75 ± 2.11 & 10.69 ± 1.16 & \textbf{11.06 ± 1.32} \\
		\hline
		\multirow{3}{*}{10\%} & LR & 23.44 ± 0.93 & \textbf{23.50 ± 1.00} & \textbf{23.50 ± 1.00} & 22.56 ± 1.54 & \textbf{23.50 ± 1.06} & 23.38 ± 1.00 \\
		& DNN & 7.50 ± 3.34 & 20.75 ± 3.01 & 19.94 ± 3.77 & 20.31 ± 2.78 & 20.50 ± 3.22 & \textbf{21.31 ± 3.77} \\
		& CNN & 15.00 ± 2.83 & 21.81 ± 3.11 & 21.75 ± 3.11 & 18.44 ± 4.37 & 22.19 ± 2.81 & \textbf{23.56 ± 3.11} \\
		\hline
		\multirow{3}{*}{15\%} & LR & \textbf{36.06 ± 0.97} & \textbf{36.06 ± 1.14} & \textbf{36.06 ± 1.14} & 35.38 ± 1.62 & 35.69 ± 1.69 & 35.13 ± 1.14 \\
		& DNN & 12.63 ± 4.62 & 32.81 ± 3.47 & 32.50 ± 3.72 & 32.19 ± 3.40 & 32.56 ± 3.61 & \textbf{33.31 ± 3.72} \\
		& CNN & 23.44 ± 4.68 & 34.19 ± 4.17 & 34.19 ± 4.17 & 29.75 ± 5.93 & 34.56 ± 3.98 & \textbf{36.31 ± 4.17} \\
		\hline
		\multirow{3}{*}{20\%} & LR & 48.63 ± 1.11 & \textbf{48.69 ± 1.16} & \textbf{48.69 ± 1.16} & 47.94 ± 1.52 & 46.56 ± 3.12 & 42.94 ± 1.16 \\
		& DNN & 22.31 ± 6.14 & 45.31 ± 3.29 & 43.94 ± 5.20 & 44.13 ± 3.64 & 45.19 ± 3.30 & \textbf{45.56 ± 5.20} \\
		& CNN & 31.00 ± 5.79 & 46.19 ± 4.33 & 46.25 ± 4.35 & 41.50 ± 7.66 & 47.13 ± 3.35 & \textbf{48.69 ± 4.35} \\
		\hline
	\end{tabular}
\end{table*}

\subsection{Influence Dynamics and Similarity across Training Stages} \label{secinflusimi}
After validating DIT's accuracy in estimating sample influence, we used it to analyze the similarity of different training stages.
The training process was adaptively divided into early, middle, and late stages using change points identified in the overall training loss trajectory. Detailed experimental settings are provided in Appendix \ref{appstages}. Then, we set time windows based on stages and used DIT to compute sample influence within these windows. We then used Kendall's tau correlation to quantify the similarity of influence rankings between stages, with higher values indicating greater stability. Table \ref{tab:influence_stability} presents these correlations.

Table \ref{tab:influence_stability} shows several key insights.
First, sample influence evolves significantly throughout training, as evidenced by the consistently low correlations between early and late stages (Early-Late column). This challenges the static influence measurement methods and highlights the necessity for time-aware methods like DIT.
Second, mid-training influence strongly correlates with full-training influence across all datasets and models. This suggests that influential samples can be identified before convergence. Mid-training analysis may suffice for estimating full-training sample influence, potentially reducing computational costs. These insights have significant implications for data selection and curriculum learning strategies. 
Third, for a given dataset, the patterns of influence ranking changes at different stages are similar across different model architectures when accounting for standard deviations. This consistency suggests that the influence of samples is largely determined by the inherent dataset rather than being heavily model-dependent.

\subsection{Applications of DIT} \label{secappdit} 
\paragraph{Flipped label Sample Detection}
To show the practical utility of DIT, we applied it to detecting flipped labels in a binary classification problem using the MNIST dataset (distinguishing between digits `1' and `7'). We randomly selected and flipped labels for 5\%, 10\%, 15\%, and 20\% of training data, corresponding to 12, 25, 38, and 51 samples, respectively. Models were then trained on these partially corrupted datasets. We calculated influence using six methods: full-process DIT, IF, LOO, and epoch-specific DIT (first, middle, and last epochs). For each method, we ranked training samples by their negative influence and evaluated the top-$k$ samples, where $k$ equals the number of deliberately flipped samples. This approach allows us to assess each method's ability to identify mislabeled samples accurately. Table \ref{tab:flipmnist} presents results averaged over 16 runs.

First, DIT consistently outperforms IF across all scenarios, often matching or closely approaching the LOO performance. DIT maintains its performance advantage across varying levels of label noise (5\% to 20\%).
Second, the performance gap between DIT and IF widens as model complexity increases (LR $<$ DNN $<$ CNN), highlighting DIT's robustness to non-convexity.
Third, later training stages generally yield better detection accuracy, particularly for complex models. As models converge, the influence of mislabeled samples becomes more distinguishable relative to correctly labeled ones.
These findings collectively show DIT's effectiveness as a powerful tool for enhancing model robustness and sample quality assessment, particularly in complex, real-world machine learning scenarios.

\paragraph{Scaling to large-scale Models}
Large-scale models present unique challenges for influence analysis. With hundreds of millions of parameters and inherent non-convexity, they fundamentally challenge traditional influence functions that rely on linear approximations. Additionally, the vast number of parameters makes computing the inverse Hessian matrix computationally infeasible.  

To validate DIT's effectiveness in large-scale model scenarios, we conducted experiments using Vision Transformer (ViT-B/16) with 86M parameters on ImageNet dataset. We evaluated DIT's capability in detecting corrupted labels across varying corruption levels. Specifically, we created corrupted versions of ImageNet by selecting two visually distinct classes (`brambling' and `daisy') and flipping the labels of $X$ brambling samples to daisy, where $X \in \{100, 150, 200, 250, 300\}$. Following the standard pre-train and then fine-tuning paradigm, we fine-tuned the ViT model for one epoch on each corrupted dataset. For each test, we applied DIT to a pool of 1000 samples, comprising $X$ flipped samples and randomly selected $(1000-X)$ benign samples. We then ranked all samples by their DIT influence value and examined how many samples with flipped labels appeared among the $X$ samples with the lowest value.

\begin{table}[ht]
	\centering
	\caption{Flipped Label Detection Results on ViT-B/16}
	\label{tab:vit_detection}
	\setlength{\tabcolsep}{3.5pt}    
	\begin{tabular}{ccc}
		\toprule
		Flipped Labels (X) & Correctly Detected & Detection Rate \\
		\midrule
		100 & 98 & 98.00\% \\
		150 & 148 & 98.67\% \\
		200 & 199 & 99.50\% \\
		250 & 247 & 98.80\% \\
		300 & 299 & 99.67\% \\
		\bottomrule
	\end{tabular}
\end{table}

The results in Table \ref{tab:vit_detection} show DIT's exceptional performance in large-scale scenarios, maintaining detection rates above 98\% across all test cases. This high accuracy, achieved with minimal computational overhead, showcases DIT's unique advantages for large-scale models:  compatibility with the pre-train then fine-tune paradigm, efficient analysis using only one epoch of fine-tuning information, and robust performance even with complex model architectures. These results validate DIT's scalability and suggest its potential as a practical tool for quality control in large-scale model fine-tuning pipelines.

\section{Related Works}
Estimating training sample influence is crucial for model optimization and interpretability. While LOO provides ground truth, it is computationally intensive. Influence functions~\cite{koh2017understanding} and recent extensions ~\cite{guo2021fastif,schioppa2022scaling,choe2024your}  offer a faster alternative, estimating the impact of removing a single training sample on model performance at convergence. However, their effectiveness is limited in non-convex and non-convergent scenarios~\cite{basu2021influence}. 

Shapley Value-based approaches ~\cite{ghorbani2019data} provide a robust, equitable valuation of individual sample contributions by considering all possible subsets of training data. Efficient approximation algorithms~\cite{jia2019towards,jia2021scalability,xu2021validation} and domain-specific extensions~\cite{schoch2022cs, sun2023shapleyfl,fan2022improving} have improved scalability, but remain computationally expensive for large-scale problems.

Data cleansing and pruning focus on removing noisy or irrelevant data. SGD-influence ~\cite{hara2019data} analyzes the gradient descent process and estimates sample influence across the entire training trajectory. Our proposed DIT extends this approach, enabling influence estimation within arbitrary time windows and providing detailed experimental evaluation. Forgetting events~\cite{toneva2018empirical} and early-training scores~\cite{paul2021deep} enable efficient data pruning. MOSO ~\cite{tan2024data} identifies less informative samples via gradient deviations, and YOCO ~\cite{he2023you} enables flexible resizing of condensed datasets.

Despite these advancements, analyzing sample influence within arbitrary time windows during training remains challenging. DIT addresses this gap by providing a computationally efficient method for dynamic influence tracking without relying on convexity and convergence assumptions. It enables multidimensional influence measurement in a single training process, offering a comprehensive understanding of sample importance.

\section{Conclusion}
This paper introduces Dynamic Influence Tracker (DIT), a novel approach for fine-grained estimation of individual training sample influence within arbitrary time windows during training. Our method's query-based design enables multi-faceted analysis of sample influence on various aspects of model performance effectively. Our theoretical analysis provides error bounds without assuming convexity and convergence. Extensive experimental results reveal patterns in influence dynamics and show that DIT consistently outperforms existing methods in influence estimation accuracy, particularly for complex models and datasets.

\bibliographystyle{IEEEtran}
\bibliography{citation}

\onecolumn
\clearpage
\appendices
\section{Detailed Derivation of Parameter Change Estimation}\label{app:detailed_derivation}
We start from Eq.(\ref{eqdiff}), which establishes the relationship:

\begin{align}\label{eq101}
	&\theta_{-j}^{[t+1]} - \theta^{[t+1]} = (\theta_{-j}^{[t]} - \theta^{[t]}) - \frac{\eta_t}{|S_t|} (\sum_{i \in S_t\setminus\{j\}} g(z_i; \theta_{-j}^{[t]}) - \sum_{i \in S_t} g(z_i; \theta^{[t]}))\\
	&= (\theta_{-j}^{[t]} - \theta^{[t]}) - \frac{\eta_t}{|S_t|} (\sum_{i \in S_t\setminus\{j\}} g(z_i; \theta_{-j}^{[t]}) -\sum_{i \in S_t\setminus\{j\}}g(z_i; \theta^{[t]})-\mathbf{1}_{j \in S_t}g(z_i; \theta^{[t]}))\\
	&= (\theta_{-j}^{[t]} - \theta^{[t]}) - \frac{\eta_t}{|S_t|} (\sum_{i \in S_t\setminus\{j\}} g(z_i; \theta_{-j}^{[t]}) -\sum_{i \in S_t\setminus\{j\}}g(z_i; \theta^{[t]}))+\frac{\eta_t}{|S_t|}\mathbf{1}_{j \in S_t}g(z_i; \theta^{[t]})\\
	&= (\theta_{-j}^{[t]} - \theta^{[t]}) - \frac{\eta_t}{|S_t|} \sum_{i \in S_t\setminus\{j\}} (g(z_i; \theta_{-j}^{[t]}) -g(z_i; \theta^{[t]}))+\frac{\eta_t}{|S_t|}\mathbf{1}_{j \in S_t}g(z_i; \theta^{[t]}),
\end{align}
where $\mathbf{1}_{j \in S_t}$ is an indicator function that equals 1 if $j \in S_t$, otherwise 0.

Using Eq.(\ref{eqht}), we have:
\begin{equation} \label{eq111}
	\sum_{i \in S_t\setminus\{j\}}( g(z_i; \theta_{-j}^{[t]}) - g(z_i; \theta^{[t]}) )\approx \sum_{i \in S_t\setminus\{j\}} \nabla_\theta g(z_i; \theta^{[t]})^T (\theta_{-j}^{[t]} - \theta^{[t]}),
\end{equation}
Following Eq.(\ref{eqg}) and Assumption (A4) detailed in Appendix \ref{app:estimation_error_proof}, we have:
\begin{equation} \label{eq121}
	\sum_{i\in S_t\backslash\{j\}} \nabla_\theta g(z_i; \theta^{[t]})^T(\theta_{-j}^{[t]} - \theta^{[t]})=|S_t| H_{-j}^{[t]}(\theta_{-j}^{[t]} - \theta^{[t]})\approx |S_t| H^{[t]}(\theta_{-j}^{[t]} - \theta^{[t]}).
\end{equation}
Combining Eq.(\ref{eq111}) and Eq.(\ref{eq121}), we have:
\begin{equation} \label{eq112}
	\sum_{i \in S_t\setminus\{j\}}( g(z_i; \theta_{-j}^{[t]}) - g(z_i; \theta^{[t]}) )\approx |S_t| H^{[t]}(\theta_{-j}^{[t]} - \theta^{[t]}).
\end{equation}
Applying Eq.~(\ref{eq112}) to Eq.~(\ref{eq101}), we have the final result:
\begin{align}\label{eq102}
	&\theta_{-j}^{[t+1]} - \theta^{[t+1]}= (\theta_{-j}^{[t]} - \theta^{[t]}) - \frac{\eta_t}{|S_t|} \sum_{i \in S_t\setminus\{j\}} (g(z_i; \theta_{-j}^{[t]}) -g(z_i; \theta^{[t]}))+\frac{\eta_t}{|S_t|}\mathbf{1}_{j \in S_t}g(z_i; \theta^{[t]}) \\
	& \approx (\theta_{-j}^{[t]} - \theta^{[t]}) - \frac{\eta_t}{|S_t|} (|S_t| H^{[t]}(\theta_{-j}^{[t]} - \theta^{[t]}) )+\frac{\eta_t}{|S_t|}\mathbf{1}_{j \in S_t}g(z_i; \theta^{[t]})\\
	&= (\theta_{-j}^{[t]} - \theta^{[t]}) - \eta_t H^{[t]}(\theta_{-j}^{[t]} - \theta^{[t]}) +\frac{\eta_t}{|S_t|}\mathbf{1}_{j \in S_t}g(z_i; \theta^{[t]})\\
	&= (I- \eta_t H^{[t]})(\theta_{-j}^{[t]} - \theta^{[t]}) +\frac{\eta_t}{|S_t|}\mathbf{1}_{j \in S_t}g(z_i; \theta^{[t]}).
\end{align}
This derivation confirms the correctness of Eq. (\ref{eq14}), including the last term.

\clearpage
\section{Estimation Error Analysis without Convexity Assumptions} \label{app:estimation_error_proof}

\begin{theorem}[Error Bound for DIT Parameter Change]
	Let $\Delta \theta_{-j}^{[t_1,t_2]}$ be the true influence of excluding sample $z_j$ on the model parameters over the interval $[t_1, t_2]$ during SGD training. Let $\widehat{\Delta \theta}_{-j}^{[t_1,t_2]}$ be its approximation using DIT. Under the following assumptions:
	\begin{enumerate}[label=(A\arabic*)]
		\item Lipschitz Continuity of Gradient: The gradient $\nabla \ell(z_i; \theta)$ is Lipschitz continuous with constant $L_g$:
		$\| \nabla \ell(z_i; \theta_1) - \nabla \ell(z_i; \theta_2) \| \leq L_g \| \theta_1 - \theta_2 \|$, $ \forall \theta_1, \theta_2 \in \Theta, \forall i$.
		\item Lipschitz Continuity of Hessian: The Hessian $\nabla^2 \ell(z_i; \theta)$ is Lipschitz continuous with constant $L_H$:
		$\| \nabla^2 \ell(z_i; \theta_1) - \nabla^2 \ell(z_i; \theta_2) \| \leq L_H \| \theta_1 - \theta_2 \|$, $ \forall \theta_1, \theta_2 \in \Theta, \forall i$.
		\item Learning Rate Bound: The learning rate satisfies $\eta_t \leq \frac{1}{L_H}$ for all $t$.
		\item Hessian Approximation Error: The Hessian approximation error is bounded: $	\| H^{[t]} - {H}_{-j}^{[t]} \| \leq \epsilon_H$, $ \forall t$,
		where ${H}_{-j}^{[t]}= \frac{1}{|S_t \setminus \{j\}|} \sum_{i \in S_t \setminus \{j\}} \nabla^2 \ell(z_i; \theta^{[t]})$ is the empirical Hessian over the mini-batch.
		\item Gradient Norm Bound: For all $\theta \in \Theta$ and all $z_i$: $\| \nabla \ell(z_i; \theta) \| \leq G$.
		\item Parameter Difference Bound: There exists a constant $M > 0$ such that: $\| \theta_{-j}^{[t]} - \theta^{[t]} \| \leq M$, $ \forall t \in [t_1, t_2]$.
		\item Bounded Hessian Norm: For all $\theta \in \Theta$ and all $ z_i$: $	\| \nabla^2 \ell(z_i; \theta) \| \leq M_H$.
	\end{enumerate}
	Then, the expected estimation error is bounded as follows:
	\begin{equation}
		\mathbb{E} \left[ \left\| \Delta \theta_{-j}^{[t_1,t_2]} - \widehat{\Delta \theta}_{-j}^{[t_1,t_2]} \right\| \right] \leq \frac{\tilde{B}}{M_H} \left( e^{M_H \eta_{\max} (t_2 + 1)} + e^{M_H \eta_{\max} (t_1 + 1)} - 2 \right)
	\end{equation}
	where: $\eta_{\max} = \max_{t \in [t_1,t_2]} \eta_t$, $\tilde{B} = \frac{L_H M^2}{2} + \epsilon_H M$, $n$ is the total number of samples in the dataset.
\end{theorem}

\begin{proof}
	\textbf{Step 1: Derivation of the Error Update Equation}
	
	Define the error at iteration $t$:
	\begin{equation}
		e^{[t]} = (\theta_{-j}^{[t]} - \theta^{[t]}) - \widehat{\Delta \theta}_{-j}^{[t]}
	\end{equation}
	where $ \widehat{\Delta \theta}_{-j}^{[t]} = \widehat{\Delta \theta}_{-j}^{[0,t]} $ is the approximation of the true parameter change $ \Delta \theta_{-j}^{[t]} $ using the DIT method.
	
	Our aim is to derive a recursive relation for $ e^{[t]} $ and then bound its expected norm.
	
	Consider the updates for $\theta^{[t]}$, $\theta_{-j}^{[t]}$, and $\widehat{\theta}_{-j}^{[t]}$:
	
	Original SGD Update:
	\begin{equation}
		\theta^{[t+1]} = \theta^{[t]} - \eta_t \tilde{g}^{[t]}, \quad \tilde{g}^{[t]} = \frac{1}{|S_t|} \sum_{i \in S_t} \nabla \ell(z_i; \theta^{[t]}).
	\end{equation}
	Leave-One-Out SGD Update:
	\begin{equation}
		\theta_{-j}^{[t+1]} = \theta_{-j}^{[t]} - \eta_t \tilde{g}_{-j}^{[t]}, \quad \tilde{g}_{-j}^{[t]} = \frac{1}{|S_t|} \sum_{i \in S_t \setminus \{j\}} \nabla \ell(z_i; \theta_{-j}^{[t]}).
	\end{equation}
	Approximate Leave-One-Out Update (DIT Method):
	\begin{equation}
		\widehat{\theta}_{-j}^{[t+1]} = \widehat{\theta}_{-j}^{[t]} - \eta_t \left( \tilde{g}^{[t]} + H^{[t]} (\widehat{\theta}_{-j}^{[t]} - \theta^{[t]}) - \mathbf{1}_{\{j \in S_t\}} \frac{1}{|S_t|} \nabla \ell(z_j; \theta^{[t]}) \right).
	\end{equation}
	
	We derive the error update equation as follows:
	\begin{align}
		e^{[t]} - e^{[t-1]} = \eta_{t-1} \delta^{[t-1]},
	\end{align}
	where:
	\begin{equation}
		\delta^{[t-1]} = \left( \tilde{g}_{-j}^{[t-1]} - \tilde{g}^{[t-1]} \right) - H^{[t-1]} \widehat{\Delta \theta}_{-j}^{[t-1]} + \mathbf{1}_{\{j \in S_{t-1}\}} \frac{1}{|S_{t-1}|} \nabla \ell(z_j; \theta^{[t-1]}).
	\end{equation}
	or equivalently:
	\begin{equation}
		\delta^{[t]} = \left( \tilde{g}_{-j}^{[t]} - \tilde{g}^{[t]} \right) - H^{[t]} \widehat{\Delta \theta}_{-j}^{[t]} + \mathbf{1}_{\{j \in S_t\}} \frac{1}{|S_t|} \nabla \ell(z_j; \theta^{[t]}).
	\end{equation}
	\textbf{Step 2: Bounding $\| \delta^{[t]} \|$}
	
	We decompose $\delta^{[t]}$ and bound each term:
	
	\textbf{1. Difference in Stochastic Gradients:}
	\begin{equation}
		\tilde{g}_{-j}^{[t]} - \tilde{g}^{[t]} = \frac{1}{|S_t|} \left( \sum_{i \in S_t \setminus \{j\}} \left( \nabla \ell(z_i; \theta_{-j}^{[t]}) - \nabla \ell(z_i; \theta^{[t]}) \right) - \mathbf{1}_{\{j \in S_t\}} \nabla \ell(z_j; \theta^{[t]}) \right).
	\end{equation}
	
	Applying a first-order Taylor expansion to $\nabla \ell(z_i; \theta_{-j}^{[t]})$ for $i \neq j$:
	\begin{equation}
		\nabla \ell(z_i; \theta_{-j}^{[t]}) - \nabla \ell(z_i; \theta^{[t]}) = \nabla^2 \ell(z_i; \theta^{[t]})(\theta_{-j}^{[t]} - \theta^{[t]}) + r_{i,j}^{[t]},
	\end{equation}
	where, by Assumption (A2):
	\begin{equation}
		\| r_{i,j}^{[t]} \| \leq \frac{L_H}{2} \| \theta_{-j}^{[t]} - \theta^{[t]} \|^2
	\end{equation}
	Thus, we have:
	\begin{align}
		\tilde{g}_{-j}^{[t]} - \tilde{g}^{[t]} & = \frac{1}{|S_t|} \sum_{i \in S_t \setminus \{j\}} \nabla^2 \ell(z_i; \theta^{[t]}) (\theta_{-j}^{[t]} - \theta^{[t]}) + r_{i,j}^{[t]} - \mathbf{1}_{\{j \in S_t\}} \nabla \ell(z_j; \theta^{[t]}) \nonumber \\
		& = \frac{1}{|S_t|} \left( \sum_{i \in S_t \setminus \{j\}} r_{i,j}^{[t]} - \mathbf{1}_{\{j \in S_t\}} \nabla \ell(z_j; \theta^{[t]}) \right) + H_{-j}^{[t]} (\theta_{-j}^{[t]} - \theta^{[t]})
	\end{align}
	
	\textbf{2. Hessian Approximation Error:}
	\begin{equation}
		\|(H_{-j}^{[t]} - H^{[t]}) (\theta_{-j}^{[t]} - \theta^{[t]}) \| \leq \epsilon_H \| \theta_{-j}^{[t]} - \theta^{[t]} \|.
	\end{equation}
	according to Assumption (A4).
	
	\textbf{3. Combining Terms:}
	Substitute the approximations back into $ \delta^{[t]} $:
	\begin{align}
		\delta^{[t]} & = \left( \tilde{g}_{-j}^{[t]} - \tilde{g}^{[t]} \right) - H^{[t]} \widehat{\Delta \theta}_{-j}^{[t]} + \mathbf{1}_{\{j \in S_t\}} \frac{1}{|S_t|} \nabla \ell(z_j; \theta^{[t]}) \nonumber \\
		& = \left( \tilde{g}_{-j}^{[t]} - \tilde{g}^{[t]} \right) - H_{-j}^{[t]} (\theta_{-j}^{[t]} - \theta^{[t]}) + \left( H_{-j}^{[t]} - H^{[t]} \right) (\theta_{-j}^{[t]} - \theta^{[t]}) + \mathbf{1}_{\{j \in S_t\}} \frac{1}{|S_t|} \nabla \ell(z_j; \theta^{[t]}) \nonumber \\
		& = \frac{1}{|S_t|} \sum_{i \in S_t \setminus \{j\}} r_{i,j}^{[t]} + (H_{-j}^{[t]} - H^{[t]}) (\theta_{-j}^{[t]} - \theta^{[t]}) + H^{[t]} ((\theta_{-j}^{[t]} - \theta^{[t]}) - \Delta \widehat{\theta}_{-j}^{[t]} ) \nonumber \\
		& = \frac{1}{|S_t|} \sum_{i \in S_t \setminus \{j\}} r_{i,j}^{[t]} + (H_{-j}^{[t]} - H^{[t]}) (\theta_{-j}^{[t]} - \theta^{[t]}) + H^{[t]} e^{[t]}.
	\end{align}
	
	\textbf{4. Bounding $\| \delta^{[t]} \|$:}
	\begin{itemize}
		\item \textbf{First Term:}
		\begin{equation}
			\left\| \frac{1}{|S_t|} \sum_{i \in S_t \setminus \{j\}} r_{i,j}^{[t]} \right\| < \frac{L_H M^2}{2}.
		\end{equation}
		\item \textbf{Second Term:}
		\begin{equation}
			\left\| (H_{-j}^{[t]} - H^{[t]}) (\theta_{-j}^{[t]} - \theta^{[t]}) \right\| \leq \epsilon_H M.
		\end{equation}
		\item \textbf{Third Term:}
		\begin{equation}
			\left\| H^{[t]} e^{[t]} \right\| \leq M_H \| e^{[t]} \|.
		\end{equation}
	\end{itemize}

	Combining bounds, we can have:
	\begin{equation}
		\| \delta^{[t]} \| < \frac{L_H M^2}{2} + \epsilon_H M + M_H \| e^{[t]} \|.
	\end{equation}
	
	\textbf{Step 3: Error Update Equation}
	
	Using the error update:
	\begin{equation}
		e^{[t]} = e^{[t-1]} - \eta_t \delta^{[t-1]},
	\end{equation}
	we have:
	\begin{equation}
		\| e^{[t]} \| \leq \| e^{[t-1]} \| + \eta_t \| \delta^{[t-1]} \| < \| e^{[t-1]} \| + \eta_t \left( \frac{L_H M^2}{2} + \epsilon_H M + M_H \| e^{[t-1]} \| \right).
	\end{equation}
	
	Define:
	\begin{equation}
		a_t = 1 + \eta_t M_H, \quad b_t = \eta_t \left( \frac{L_H M^2}{2} + \epsilon_H M \right).
	\end{equation}
	
	Then:
	\begin{equation}
		\| e^{[t]} \| < a_t \| e^{[t-1]} \| + b_t.
	\end{equation}
	
	\textbf{Step 4: Taking Expectations}
	
	Taking expectations over the mini-batch sampling:
	\begin{equation}
		\mathbb{E} \left[ \| e^{[t]} \| \right] < a_t \mathbb{E} \left[ \| e^{[t-1]} \| \right] + b_t.
	\end{equation}
	
	Define:
	\begin{equation}
		\tilde{B} = \frac{L_H M^2}{2} + \epsilon_H M.
	\end{equation}
	Then:
	\begin{equation}
		\mathbb{E} \left[ \| e^{[t]} \| \right] < a_t \mathbb{E} \left[ \| e^{[t-1]} \| \right] + \eta_t \tilde{B}.
	\end{equation}
	
	\textbf{Step 5: Solving the Recurrence Relation}
	
	Unfolding the recurrence:
	\begin{equation}
		\mathbb{E} \left[ \| e^{[t]} \| \right] \leq \prod_{k=0}^{t} a_k \cdot \mathbb{E} \left[ \| e^{[0]} \| \right] + \sum_{s=0}^{t} \left( \prod_{k=s+1}^{t} a_k \right) b_s.
	\end{equation}
	Since $ e^{[0]} = 0 $, we have:
	\begin{equation}
		\mathbb{E} \left[ \| e^{[t]} \| \right] \leq \sum_{s=0}^{t} \left( \prod_{k=s+1}^{t} a_k \right) b_s.
	\end{equation}
	
	Assuming $ a_k \leq e^{M_H \eta_{\max}} $, we get:
	\begin{equation}
		\prod_{k=s+1}^{t} a_k \leq e^{M_H \eta_{\max} (t - s)}.
	\end{equation}
	
	Therefore:
	\begin{equation}
		\mathbb{E} \left[ \| e^{[t]} \| \right] \leq \tilde{B} \eta_{\max} \sum_{s=0}^{t} e^{M_H \eta_{\max} (t - s)}.
	\end{equation}
	
	Approximating the sum:
	\begin{equation}
		\mathbb{E} \left[ \| e^{[t]} \| \right] \leq \tilde{B} \eta_{\max} \cdot \frac{e^{M_H \eta_{\max} (t + 1)} - 1}{e^{M_H \eta_{\max}} - 1}.
	\end{equation}
	For small $ M_H \eta_{\max} $, $ e^{M_H \eta_{\max}} - 1 \approx M_H \eta_{\max} $, yielding:
	\begin{equation}
		\mathbb{E} \left[ \| e^{[t]} \| \right] \leq \frac{\tilde{B}}{M_H} \left( e^{M_H \eta_{\max} (t + 1)} - 1 \right).
	\end{equation}
	Substitute $t$ with $t_1$ and $t_2$ respectively:
	
	\begin{align}
		\mathbb{E} \left[ \| e^{[t_2]} \| \right] & \leq \frac{\tilde{B}}{M_H} \left( e^{M_H \eta_{\max} (t_2 + 1)} - 1 \right), \\
		\mathbb{E} \left[ \| e^{[t_1]} \| \right] & \leq \frac{\tilde{B}}{M_H} \left( e^{M_H \eta_{\max} (t_1 + 1)} - 1 \right).
	\end{align}
	
	\textbf{Step 6: Final Bound}
	
	The estimation error is:
	\begin{align}
		\mathbb{E} \left[ \left\| \Delta \theta_{-j}^{[t_1,t_2]} - \widehat{\Delta \theta}_{-j}^{[t_1,t_2]} \right\| \right] & \leq \mathbb{E} \left[ \| e^{[t_2]} \| \right] + \mathbb{E} \left[ \| e^{[t_1]} \| \right] \nonumber \\
		& \leq \frac{\tilde{B}}{M_H} \left( e^{M_H \eta_{\max} (t_2 + 1)} + e^{M_H \eta_{\max} (t_1 + 1)} - 2 \right)
	\end{align}
	
	This completes the proof.
\end{proof}
\begin{remark}
	Note that DIT applies to non-converged and non-convex models. The exponential form arises from the recursive nature of error propagation, where each SGD step compounds previous errors multiplicatively. Our analysis is the first to guarantee error bounds for non-converged, non-convex models during arbitrary time windows. The bounds are mathematical guarantees for the worst case, and experimental results show that DIT can achieve near-zero errors empirically.
\end{remark}
\begin{remark}
	The error bound provides several key insights:
	\begin{itemize}
		\item The error grows at most exponentially with both $t_1$ and $t_2$, highlighting the challenge of long-range influence estimation. The impact of $t_2$ is generally more significant as it represents the end of the time window.
		\item The Hessian approximation error $\epsilon_H$ directly impacts the overall error, emphasizing the importance of accurate Hessian estimation.
		\item The maximum learning rate $\eta_{\max}$ affects the error bound exponentially, suggesting that smaller learning rates might help control the estimation error.
		\item The bound depends on the Lipschitz constants of the gradient and Hessian ($L_g$ and $L_H$), indicating that smoother loss landscapes lead to more reliable influence estimates.
	\end{itemize}
	This theorem provides a theoretical foundation for understanding the limitations of influence estimation without assuming convexity and guides practical considerations in its application to large-scale machine learning problems.
\end{remark}

\clearpage

\section{DIT Toolkit}\label{app:dit_approximation}
The flexibility of query-based DIT allows for its application to a wide range of machine learning challenges. In this section, we provide a toolkit of query vectors that enables targeted investigations into critical aspects of model behavior, including gradient value, prediction changes, feature importance, and parameter importance. 
\subsection{DIT for Loss Value}
\begin{theorem}[DIT for Loss Value] \label{theoloss}
	Given a loss function $\ell(z; \theta)$, a time window $[t_1, t_2]$, a training sample $z_j$, and a test sample $z_{\text{test}}$, the Dynamic Influence Tracker with query function $q(t) = (\nabla_\theta \ell(z_{\text{test}}; \theta^{[t]})$ can be approximated as:
	\begin{equation}
		Q_{-j}^{[t_1,t_2]}(q) \approx [\ell(z_{\text{test}}; \theta_{-j}^{[t_2]}) - \ell(z_{\text{test}}; \theta_{-j}^{[t_1]})] - [\ell(z_{\text{test}}; \theta^{[t_2]}) - \ell(z_{\text{test}}; \theta^{[t_1]})],
		\label{eq:dit_loss_approx}
	\end{equation}
	where $\theta_{-j}^{[t]}$ denotes the model parameters at time $t$ when trained without sample $z_j$, and $\theta^{[t]}$ denotes the parameters when trained with all samples.
\end{theorem}
\begin{proof}
	We begin with the definition of the Query-Based Dynamic Influence Tracker:
	\begin{equation}		\label{eq:dit_def_loss}
		Q_{-j}^{[t_1,t_2]}(q) = \left\langle q(t_2), \Delta\theta_{-j}^{[t_2]} \right\rangle - \left\langle q(t_1), \Delta\theta_{-j}^{[t_1]} \right\rangle
	\end{equation}
	where $\Delta\theta_{-j}^{[t]} = \theta_{-j}^{[t]} - \theta^{[t]}$.
	
	Substituting $q(t) = \nabla_\theta \ell(z_{\text{test}}; \theta^{[t]})$ into Eq. (\ref{eq:dit_def_loss}):
	\begin{equation}\label{eq:dit_expanded}
		Q_{-j}^{[t_1,t_2]}(q) = \left\langle \nabla_\theta \ell(z_{\text{test}}; \theta^{[t_2]}), \theta_{-j}^{[t_2]} - \theta^{[t_2]} \right\rangle - \left\langle \nabla_\theta \ell(z_{\text{test}}; \theta^{[t_1]}), \theta_{-j}^{[t_1]} - \theta^{[t_1]} \right\rangle. 
	\end{equation}
	
	Apply the first-order Taylor expansion of $\ell(z_{\text{test}}; \theta)$ around $\theta^{[t_2]}$ and $\theta^{[t_1]}$:
	\begin{align}
		\ell(z_{\text{test}}; \theta_{-j}^{[t_2]}) &\approx \ell(z_{\text{test}}; \theta^{[t_2]}) + \langle \nabla_\theta \ell(z_{\text{test}}; \theta^{[t_2]}), \theta_{-j}^{[t_2]} - \theta^{[t_2]} \rangle \label{eq:taylor_expansions1} \\
		\ell(z_{\text{test}}; \theta_{-j}^{[t_1]}) &\approx \ell(z_{\text{test}}; \theta^{[t_1]}) + \langle \nabla_\theta \ell(z_{\text{test}}; \theta^{[t_1]}), \theta_{-j}^{[t_1]} - \theta^{[t_1]} \rangle
		\label{eq:taylor_expansions2}
	\end{align}
	
	Rearranging Eq. (\ref{eq:taylor_expansions1}) and Eq. (\ref{eq:taylor_expansions2}):
	\begin{align}
		\langle \nabla_\theta \ell(z_{\text{test}}; \theta^{[t_2]}), \theta_{-j}^{[t_2]} - \theta^{[t_2]} \rangle &\approx \ell(z_{\text{test}}; \theta_{-j}^{[t_2]}) - \ell(z_{\text{test}}; \theta^{[t_2]}) \\
		\langle \nabla_\theta \ell(z_{\text{test}}; \theta^{[t_1]}), \theta_{-j}^{[t_1]} - \theta^{[t_1]} \rangle &\approx \ell(z_{\text{test}}; \theta_{-j}^{[t_1]}) - \ell(z_{\text{test}}; \theta^{[t_1]})
		\label{eq:rearranged_taylor}
	\end{align}
	Substituting these approximations back into Eq. (\ref{eq:dit_expanded}):
	\begin{align}
		Q_{-j}^{[t_1,t_2]}(q) &\approx [\ell(z_{\text{test}}; \theta_{-j}^{[t_2]}) - \ell(z_{\text{test}}; \theta^{[t_2]})] - [\ell(z_{\text{test}}; \theta_{-j}^{[t_1]}) - \ell(z_{\text{test}}; \theta^{[t_1]})] \\
		&= [\ell(z_{\text{test}}; \theta_{-j}^{[t_2]}) - \ell(z_{\text{test}}; \theta_{-j}^{[t_1]})] - [\ell(z_{\text{test}}; \theta^{[t_2]}) - \ell(z_{\text{test}}; \theta^{[t_1]})]
		\label{eq:dit_approx_final}
	\end{align}
	This completes the proof of Theorem \ref{theoloss}.
\end{proof}
This theorem provides a foundation for understanding how individual training samples affect the model's loss on specific test points over time. The right-hand side of Eq. (\ref{eq:dit_loss_approx}) represents the difference between the loss changes with and without sample $z_j$, offering a direct measure of the sample's influence on model performance.

\textbf{Extension to Test Sets:}
We can extend this concept to consider an entire test set $D_{\text{test}} = \{z_1, \ldots, z_M\}$. Define the query function as:
\begin{equation} 
	q(t) = \frac{1}{M} \sum_{i=1}^M \nabla_\theta \ell(z_i; \theta^{[t]}), \quad z_i \in D_{\text{test}}. 
\end{equation}
With this choice, the DIT approximates the change in average test loss:
\begin{equation} 		\label{eq:dit_test_set}
	\begin{split} 
		Q_{-j}^{[t_1,t_2]}(q) &\approx \frac{1}{M} \sum_{i=1}^M \left[ \ell(z_i; \theta_{-j}^{[t_2]}) - \ell(z_i; \theta_{-j}^{[t_1]}) \right] - \frac{1}{M} \sum_{i=1}^M \left[ \ell(z_i; \theta^{[t_2]}) - \ell(z_i; \theta^{[t_1]}) \right] \\ 
		&= \left[ \mathcal{L}_{\text{test}}(\theta_{-j}^{[t_2]}) - \mathcal{L}_{\text{test}}(\theta_{-j}^{[t_1]}) \right] - \left[\mathcal{L}_{\text{test}}(\theta^{[t_2]}) - \mathcal{L}_{\text{test}}(\theta^{[t_1]}) \right], 
	\end{split} 
\end{equation}
where $\mathcal{L}_{\text{test}}(\theta^{[t]})=\frac{1}{M} \sum_{i=1}^M \ell(z_i; \theta^{[t]}) $ is the average test loss.

\subsection{DIT for Prediction Changes}
\begin{theorem}[DIT for Prediction Changes] \label{theopre}
	Given a model function $f(x; \theta)$, a time window $[t_1, t_2]$, a training sample $z_j$, and a test input $x_{\text{test}}$, the Dynamic Influence Tracker with query function $q(t) = \nabla_\theta f(x_{\text{test}}; \theta^{[t]})$ can be approximated as:
	\begin{equation} \label{eq:dit_prediction_approx}
		Q_{-j}^{[t_1,t_2]}(q) \approx \left[ f(x_{\text{test}}; \theta_{-j}^{[t_2]}) - f(x_{\text{test}}; \theta_{-j}^{[t_1]}) \right] - \left[ f(x_{\text{test}}; \theta^{[t_2]}) - f(x_{\text{test}}; \theta^{[t_1]}) \right],
	\end{equation}
	where $\theta_{-j}^{[t]}$ denotes the model parameters at time $t$ when trained without sample $z_j$, and $\theta^{[t]}$ denotes the parameters when trained with all samples.
\end{theorem}
\begin{proof}
	We begin with the definition of the Query-Based Dynamic Influence Tracker:
	\begin{equation}		\label{eq:dit_def}
		Q_{-j}^{[t_1,t_2]}(q) = \left\langle q(t_2), \Delta\theta_{-j}^{[t_2]} \right\rangle - \left\langle q(t_1), \Delta\theta_{-j}^{[t_1]} \right\rangle
	\end{equation}
	where $\Delta\theta_{-j}^{[t]} = \theta_{-j}^{[t]} - \theta^{[t]}$.
	
	Substituting $q(t) = \nabla_\theta f(z_{\text{test}}; \theta^{[t]})$ into Eq. (\ref{eq:dit_def}):
	\begin{equation}\label{eq:dit_expanded_pre}
		Q_{-j}^{[t_1,t_2]}(q) = \left\langle \nabla_\theta f(z_{\text{test}}; \theta^{[t_2]}), \theta_{-j}^{[t_2]} - \theta^{[t_2]} \right\rangle - \left\langle \nabla_\theta f(z_{\text{test}}; \theta^{[t_1]}), \theta_{-j}^{[t_1]} - \theta^{[t_1]} \right\rangle. 
	\end{equation}
	
	We apply the first-order Taylor approximation of the model function around $\theta^{[t_2]}$ and $\theta^{[t_1]}$:
	\begin{align}
		f(x_{\text{test}}; \theta_{-j}^{[t_2]}) &\approx f(x_{\text{test}}; \theta^{[t_2]}) + \langle \nabla_\theta f(x_{\text{test}}; \theta^{[t_2]}), \theta_{-j}^{[t_2]} - \theta^{[t_2]} \rangle \\
		f(x_{\text{test}}; \theta_{-j}^{[t_1]}) &\approx f(x_{\text{test}}; \theta^{[t_1]}) + \langle \nabla_\theta f(x_{\text{test}}; \theta^{[t_1]}), \theta_{-j}^{[t_1]} - \theta^{[t_1]} \rangle
		\label{eq:taylor_prediction}
	\end{align}
	Rearranging these equations:
	\begin{align}
		\langle \nabla_\theta f(x_{\text{test}}; \theta^{[t_2]}), \theta_{-j}^{[t_2]} - \theta^{[t_2]} \rangle &\approx f(x_{\text{test}}; \theta_{-j}^{[t_2]}) - f(x_{\text{test}}; \theta^{[t_2]}) \\
		\langle \nabla_\theta f(x_{\text{test}}; \theta^{[t_1]}), \theta_{-j}^{[t_1]} - \theta^{[t_1]} \rangle &\approx f(x_{\text{test}}; \theta_{-j}^{[t_1]}) - f(x_{\text{test}}; \theta^{[t_1]})
		\label{eq:rearranged_taylor_prediction}
	\end{align}
	Substituting these approximations back into Eq. (\ref{eq:dit_expanded_pre}):
	\begin{align}
		Q_{-j}^{[t_1,t_2]}(q) &\approx [f(x_{\text{test}}; \theta_{-j}^{[t_2]}) - f(x_{\text{test}}; \theta^{[t_2]})] - [f(x_{\text{test}}; \theta_{-j}^{[t_1]}) - f(x_{\text{test}}; \theta^{[t_1]})] \\
		&= [f(x_{\text{test}}; \theta_{-j}^{[t_2]}) - f(x_{\text{test}}; \theta_{-j}^{[t_1]})] - [f(x_{\text{test}}; \theta^{[t_2]}) - f(x_{\text{test}}; \theta^{[t_1]})]
		\label{eq:dit_prediction_approx_final}
	\end{align}
	This completes the proof of Theorem \ref{theopre}.
\end{proof}
This theorem provides a formal justification for using DIT to analyze how excluding sample $z_j$ influences the model's predictions on a test input $x_{\text{test}}$ over the interval $[t_1, t_2]$. Compared to Theorem \ref{theoloss}, which focuses on the loss value, Theorem \ref{theopre} focuses on specific model outputs. It enables the identification of influential training samples for specific predictions, aids in understanding model behavior on particular inputs, and can help detect potential outliers or mislabeled data.

\subsection{DIT for Feature Importance}
\begin{theorem}[DIT for Feature Importance]
	Given a loss function $\ell(z; \theta)$, a training sample $z = (x, y)$, and a test sample $z_\text{test} = (x_\text{test}, y_\text{test})$, the Dynamic Influence Tracker for Feature Importance with query function $q(t) = \nabla_x \nabla_\theta \ell(z_\text{test}; \theta^{[t]})$ can be approximated as:
	\begin{equation}
		Q_{-j}^{[t_1,t_2]}(q) \approx [\nabla_x \ell(z_\text{test}; \theta_{-j}^{[t_2]}) - \nabla_x \ell(z_\text{test}; \theta_{-j}^{[t_1]})] - [\nabla_x \ell(z_\text{test}; \theta^{[t_2]}) - \nabla_x \ell(z_\text{test}; \theta^{[t_1]})],
		\label{eq:dit_feature_importance_approx}
	\end{equation}
	where $\theta_{-j}^{[t]}$ denotes the model parameters at time $t$ when trained without sample $z_j$, and $\theta^{[t]}$ denotes the parameters when trained with all samples.
\end{theorem}
\begin{proof}
	We start with the definition of the Query-Based Dynamic Influence Tracker:
	\begin{equation}		\label{eq:dit_def_feature_importance}
		Q_{-j}^{[t_1,t_2]}(q) = \left\langle q(t_2), \Delta\theta_{-j}^{[t_2]} \right\rangle - \left\langle q(t_1), \Delta\theta_{-j}^{[t_1]} \right\rangle, 
	\end{equation}
	where $\Delta\theta_{-j}^{[t]} = \theta_{-j}^{[t]} - \theta^{[t]}$.
	
	Substituting $q(t) = \nabla_x \nabla_\theta \ell(z_\text{test}; \theta^{[t]})$:
	\begin{equation}
		Q_{-j}^{[t_1,t_2]}(q) = \left\langle \nabla_\theta \nabla_x \ell(z_{\text{test}}; \theta^{[t_2]}), \theta_{-j}^{[t_2]} - \theta^{[t_2]} \right\rangle - \left\langle \nabla_\theta \nabla_x \ell(z_{\text{test}}; \theta^{[t_1]}), \theta_{-j}^{[t_1]} - \theta^{[t_1]} \right\rangle. 
		\label{eq:dit_feature_importance_expanded}
	\end{equation}
	
	We apply the first-order Taylor approximation of $\nabla_x \ell(z_\text{test}; \theta)$ around $\theta^{[t_2]}$ and $\theta^{[t_1]}$:
	\begin{align} \nabla_x \ell(z_{\text{test}}; \theta_{-j}^{[t_2]}) &\approx \nabla_x \ell(z_{\text{test}}; \theta^{[t_2]}) + \nabla_\theta \nabla_x \ell(z_{\text{test}}; \theta^{[t_2]}) \left( \theta_{-j}^{[t_2]} - \theta^{[t_2]} \right), \label{eqfe1} \\
		\nabla_x \ell(z_{\text{test}}; \theta_{-j}^{[t_1]}) &\approx \nabla_x \ell(z_{\text{test}}; \theta^{[t_1]}) + \nabla_\theta \nabla_x \ell(z_{\text{test}}; \theta^{[t_1]}) \left( \theta_{-j}^{[t_1]} - \theta^{[t_1]} \right). \label{eqfe2} 
	\end{align}
	Rearranging these equations:
	\begin{align} 
		\left\langle \nabla_\theta \nabla_x \ell(z_{\text{test}}; \theta^{[t_2]}), \theta_{-j}^{[t_2]} - \theta^{[t_2]} \right\rangle &\approx \nabla_x \ell(z_{\text{test}}; \theta_{-j}^{[t_2]}) - \nabla_x \ell(z_{\text{test}}; \theta^{[t_2]}), \label{eqrearranged_taylor_feature_importance1} \\
		\left\langle \nabla_\theta \nabla_x \ell(z_{\text{test}}; \theta^{[t_1]}), \theta_{-j}^{[t_1]} - \theta^{[t_1]} \right\rangle &\approx \nabla_x \ell(z_{\text{test}}; \theta_{-j}^{[t_1]}) - \nabla_x \ell(z_{\text{test}}; \theta^{[t_1]}). \label{eq:rearranged_taylor_feature_importance}
	\end{align}
	
	Substituting these approximations back into Eq.(\ref{eq:dit_feature_importance_expanded}):
	\begin{align} 
		Q_{-j}^{[t_1,t_2]}(q) &\approx \left[ \nabla_x \ell(z_{\text{test}}; \theta_{-j}^{[t_2]}) - \nabla_x \ell(z_{\text{test}}; \theta^{[t_2]}) \right] - \left[ \nabla_x \ell(z_{\text{test}}; \theta_{-j}^{[t_1]}) - \nabla_x \ell(z_{\text{test}}; \theta^{[t_1]}) \right] \nonumber \\ 
		&= \left[ \nabla_x \ell(z_{\text{test}}; \theta_{-j}^{[t_2]}) - \nabla_x \ell(z_{\text{test}}; \theta_{-j}^{[t_1]}) \right] - \left[ \nabla_x \ell(z_{\text{test}}; \theta^{[t_2]}) - \nabla_x \ell(z_{\text{test}}; \theta^{[t_1]}) \right]. 	\label{eq:dit_feature_importance_approx_final}
	\end{align}
	This completes the proof.
\end{proof}

This theorem shows how DIT measures the impact of excluding a training sample $z_j$ on the gradient of the loss with respect to the input features at the test point $z_\text{test}$ over the interval $[t_1, t_2]$. This provides insights into how the importance of different input features evolves during training and how individual training samples influence this feature importance.

\subsection{DIT for Parameter Importance}
\begin{theorem}[DIT for Parameter Importance]
	Given a model with parameters $\theta \in \mathbb{R}^p$, a time window $[t_1, t_2]$, a training sample $z_j$, and the $i$-th standard basis vector $e_i \in \mathbb{R}^p$, the Dynamic Influence Tracker with query function $q(t) = (e_i)$ is exactly:
	\begin{equation} 
		Q_{-j}^{[t_1,t_2]}(q) = \left( \theta_{-j,i}^{[t_2]} - \theta_{-j,i}^{[t_1]} \right) - \left( \theta_i^{[t_2]} - \theta_i^{[t_1]} \right), \label{eq:dit_parameter_importance}
	\end{equation} 
	where $\theta_{-j,i}^{[t]}$ denotes the $i$-th component of the model parameters at time $t$ when trained without sample $z_j$, and $\theta_i^{[t]}$ denotes the $i$-th component of the parameters when trained with all samples.
\end{theorem}

\begin{proof}
	We start with the definition of the Query-Based Dynamic Influence Tracker:
	\begin{equation}		\label{eq:dit_def_parameter}
		Q_{-j}^{[t_1,t_2]}(q) = \left\langle q(t_2), \Delta\theta_{-j}^{[t_2]} \right\rangle - \left\langle q(t_1), \Delta\theta_{-j}^{[t_1]} \right\rangle, 
	\end{equation}
	where $\Delta\theta_{-j}^{[t]} = \theta_{-j}^{[t]} - \theta^{[t]}$.
	
	Substituting $q(t) = e_i$, which is constant over time:
	\begin{equation} 
		Q_{-j}^{[t_1,t_2]}(q) = \left\langle e_i, \theta_{-j}^{[t_2]} - \theta^{[t_2]} \right\rangle - \left\langle e_i, \theta_{-j}^{[t_1]} - \theta^{[t_1]} \right\rangle. 	\label{eq:dit_parameter_expanded}
	\end{equation}
	Since $e_i$ is the $i$-th standard basis vector, the inner product selects the $i$-th component: 
	\begin{equation} 
		Q_{-j}^{[t_1,t_2]}(q) = \left( \theta_{-j,i}^{[t_2]} - \theta_i^{[t_2]} \right) - \left( \theta_{-j,i}^{[t_1]} - \theta_i^{[t_1]} \right) 
		= \left( \theta_{-j,i}^{[t_2]} - \theta_{-j,i}^{[t_1]} \right) - \left( \theta_i^{[t_2]} - \theta_i^{[t_1]} \right). 
	\end{equation}
	This matches the expression in Eq. (\ref{eq:dit_parameter_importance}), completing our proof.
\end{proof}
This theorem allows us to isolate the influence of a training sample $z_j$ on specific model parameters over the interval $[t_1, t_2]$. A large absolute value of $Q_{-j}^{[t_1,t_2]}(q)$ indicates that sample $z_j$ has a significant influence on the $i$-th parameter during the specified time window. This is particularly useful for identifying which parameters are most affected by specific training samples and understanding the localized effects of training samples on the model.

By analyzing how $Q_{-j}^{[t_1,t_2]}(q)$ changes over different time windows, we can understand how the influence of a training sample on specific parameters evolves throughout the training process.

\clearpage
\subsection{Proof of Algorithm 2} \label{appenproofalg2}
We begin by recalling the definition:
\begin{equation}
	\label{eq:Q_definition}
	Q_{-j}^{[t_1,t_2]}(q) = \langle q(t_2), \Delta\theta_{-j}^{[t_2]} \rangle - \langle q(t_1), \Delta\theta_{-j}^{[t_1]} \rangle
\end{equation}
where $\Delta \theta_{-j}^{[t]} \approx \sum_{s=0}^{t-1} \left(\prod_{k=s+1}^{t-1} Z_k\right) \mathbf{\tilde{1}}_{j}^{[s]}$,
and $Z_t = I - \eta_{t} H^{[t]}$, $\mathbf{\tilde{1}}_{j}^{[t]} = \mathbf{1}_{j \in S_{t}}\frac{\eta_{t}}{|S_{t}|} g(z_j; \theta^{[t]})$.

Note that $Z_t$ is self-adjoint matrix, adhering to $\langle x, Z_t y\rangle = \langle Z_t x, y\rangle$ for all vectors $x, y$.

According to the update rules for $u_1$ and $u_2$ in the algorithm:
\begin{equation}
	\label{eq:update_rule}
	u_i^{[t-1]} = u_i^{[t]} - \eta_t H^{[t]} u_i^{[t]} = (I - \eta_t H^{[t]})u_i^{[t]} = Z_t u_i^{[t]}, \quad i \in \{1,2\}
\end{equation}
By recursive application of this update rule, we obtain for $s < t$:
\begin{equation}
	\label{eq:recursive_update}
	u_i^{[s]} = \left(\prod_{k=s+1}^{t-1} Z_k\right) u_i^{[t]}, \quad i \in \{1,2\}
\end{equation}
According to the accumulation of $Q$ in the algorithm, at each time step $t$, if $j \in S_t$, we have:
\begin{equation}
	\label{eq:Delta_Q}
	\Delta Q_t = \left\langle (u_2^{[t]} - u_1^{[t]}), \dfrac{\eta_t}{|S_t|} g(z_j; \theta^{[t]}) \right\rangle
\end{equation}
The algorithm initializes $u_2^{[t_2-1]} = q(t_2)$ and sets $u_1^{[t_1-1]} = q(t_1)$ at time $t_1$. Importantly, $u_1$ is not updated beyond $t_1$. Using the result from Eq. (\ref{eq:recursive_update}), we can express $u_2^{[t]}$ and $u_1^{[t]}$ as:
\begin{equation}
	\label{eq:u2_expression}
	u_2^{[t]} = \prod_{k=t+1}^{t_2-1} Z_k q(t_2), \quad \text{for } 0 \leq t < t_2
\end{equation}
\begin{equation}
	\label{eq:u1_expression}
	u_1^{[t]} = \begin{cases}
		\prod_{k=t+1}^{t_1-1} Z_k q(t_1) & \text{for } 0 \leq t < t_1 \\
		0 & \text{for } t_1 \leq t < t_2
	\end{cases}
\end{equation}
Note that $u_1^{[t]} = 0$ for $t_1 \leq t < t_2$ because $u_1$ is not updated beyond $t_1$, effectively removing its contribution to $\Delta Q_t$ in this range.

Substituting these expressions into Eq.(\ref{eq:Delta_Q}):
\begin{equation}
	\label{eq:Delta_Q_expanded}
	\Delta Q_t = \begin{cases}
		\left\langle \prod_{k=t+1}^{t_2-1} Z_k q(t_2) - \left(\prod_{k=t+1}^{t_1-1} Z_k q(t_1)\right), \mathbf{\tilde{1}}_j^{[t]} \right\rangle & \text{for } 0 \leq t < t_1 \\
		\left\langle \prod_{k=t+1}^{t_2-1} Z_k q(t_2), \mathbf{\tilde{1}}_j^{[t]} \right\rangle & \text{for } t_1 \leq t < t_2
	\end{cases}
\end{equation}
The total $Q$ is the sum of all $\Delta Q_t$: $Q = \sum_{t=0}^{t_2-1} \Delta Q_t$.

Expanding this sum and recalling that $Z_t$ is self-adjoint, we get:
\begin{equation}
	\label{eq:Q_expanded}
	Q = \left\langle q(t_2), \sum_{t=0}^{t_2-1} \left(\prod_{k=t+1}^{t_2-1} Z_k\right) \mathbf{\tilde{1}}_j^{[t]} \right\rangle - \left\langle q(t_1), \sum_{t=0}^{t_1-1} \left(\prod_{k=t+1}^{t_1-1} Z_k\right) \mathbf{\tilde{1}}_j^{[t]} \right\rangle
\end{equation}
Note that $u_2^{[t]}$ contributes to the first term over the entire interval $[0, t_2)$, while $u_1^{[t]}$ only contributes to the second term over $[0, t_1)$. This distinction arises from the algorithm's design, where $u_1$ is not updated beyond $t_1$.

Combined Eq. (\ref{eq:Q_expanded}) are precisely the definitions of $\Delta\theta_{-j}^{[t_2]}$ and $\Delta\theta_{-j}^{[t_1]}$, we have:
\begin{equation}
	\label{eq:Q_final}
	Q = \langle q(t_2), \Delta\theta_{-j}^{[t_2]} \rangle - \langle q(t_1), \Delta\theta_{-j}^{[t_1]} \rangle = Q_{-j}^{[t_1,t_2]}(q)
\end{equation}

Thus, we have rigorously demonstrated that the algorithm's output $Q$ is equivalent to the defined $Q_{-j}^{[t_1,t_2]}(q)$ in Eq. (\ref{eq:Q_definition}) under the stated assumption on $\eta_t$.

\clearpage
\section{Checkpoint-based Implementation of DIT}\label{app:dit_Checkpoints}
To balance storage overhead and computational efficiency, we propose a checkpoint-based implementation of DIT. This implementation significantly reduces storage requirements while maintaining the ability to compute accurate influence values for any time window.

Instead of storing parameters at every training step, we store checkpoints at regular intervals (e.g., epoch boundaries) along with essential training metadata (batch indices and learning rates).  When computing influence for a time window $[t_{1},t_{2}]$, we efficiently recover necessary parameters by loading the nearest checkpoint before $t_{1}$, reconstructing the parameter trajectory up to $t_{2}$, and storing intermediate parameters required for influence computation. The checkpoint interval provides a configurable trade-off between storage overhead and computational cost. More frequent checkpoints reduce recomputation but increase storage, while fewer checkpoints save storage at the cost of more recomputation.

\begin{figure}[H]
	\centering
	\begin{minipage}[t]{.48\textwidth}
		\begin{algorithm}[H]
			\caption{DIT Training with Checkpoints}\label{alg:dit_checktraining}
			\begin{algorithmic}[1]
				\Require Training dataset $D = \{z_n\}_{n=1}^N$, learning rate $\eta_t$, batch size $|S_t|$, training steps $T$, checkpoint interval $C$
				\Ensure Stored checkpoints and metadata $M$
				\State Initialize model parameters $\theta^{[0]}$
				\State Initialize metadata storage $M \gets \{\}$ \Comment{Store checkpoints, batch indices, learning rates}
				\For {$t = 0$ \textbf{to} $T-1$}
				\State $S_t \gets \text{SampleBatch}(D, |S_t|)$
				\State $\theta^{[t+1]} \gets \theta^{[t]} - \frac{\eta_t}{|S_t|} \sum_{i \in S_t} g(z_i; \theta^{[t]})$
				\State $M.indices[t] \gets S_t$ \Comment{Store batch indices}
				\State $M.lr[t] \gets \eta_t$ \Comment{Store learning rate}
				\If{$t \bmod C = 0$ \textbf{or} $t = T-1$}
				\State $M.checkpoints[t] \gets \theta^{[t+1]}$ \Comment{Store checkpoint}
				\EndIf
				\EndFor
				\State \textbf{return} $M$
			\end{algorithmic}
		\end{algorithm}
	\end{minipage}%
	\hfill
	\begin{minipage}[t]{.48\textwidth}	
		\begin{algorithm}[H]
			\caption{DIT Sample Influence with Checkpoints}\label{alg:dit_checkinfluence}
			\begin{algorithmic}[1]
				\Require Metadata $M$, query function $q$, time window $[t_1, t_2]$, sample $z_j$
				\Ensure Estimated influence $Q$
				\State $Q \gets 0$, $u_1^{[t_2-1]} \gets 0$, $u_2^{[t_2-1]} \gets q(t_2)$
				\State $c_1 \gets \max\{t : t \leq t_1 \text{ and } t \in M.checkpoints\}$ \Comment{Find nearest checkpoint before $t_1$}
				\State $\theta^{[c_1]} \gets M.checkpoints[c_1]$ 
				\State \Comment{Compute and store all necessary parameters from checkpoint to $t_2$}
				\State Initialize parameter storage $P \gets \{\}$
				\For {$t = c_1$ \textbf{to} $t_2-1$}
				\State $S_t \gets M.indices[t]$
				\State $\eta_t \gets M.lr[t]$
				\If{$t \in M.checkpoints$}
				\State $\theta^{[t]} \gets M.checkpoints[t]$
				\EndIf
				\State $\theta^{[t+1]} \gets \theta^{[t]} - \frac{\eta_t}{|S_t|} \sum_{i \in S_t} g(z_i; \theta^{[t]})$
				\State $P[t] \gets \theta^{[t]}$ \Comment{Store parameter for influence computation}
				\EndFor
				\For {$t = t_2-1$ \textbf{downto} $t_1$} 
				\If{$j \in M.indices[t]$}
				\State $Q \gets Q + \left\langle (u_2^{[t]} - u_1^{[t]}), \frac{M.lr[t]}{|M.indices[t]|} g(z_j; P[t]) \right\rangle$
				\EndIf 
				\State $H^{[t]} \gets \frac{1}{|M.indices[t]|} \sum_{i \in M.indices[t]} \nabla_\theta g(z_i; P[t])$
				\State $u_1^{[t-1]} \gets u_1^{[t]} - M.lr[t] H^{[t]} u_1^{[t]}$
				\State $u_2^{[t-1]} \gets u_2^{[t]} - M.lr[t] H^{[t]} u_2^{[t]}$
				\If{$t = t_1$}
				\State $u_1^{[t-1]} \gets q(t_1)$
				\EndIf
				\EndFor
				\State \textbf{return} $Q$
			\end{algorithmic}
		\end{algorithm}
	\end{minipage}
\end{figure}

\clearpage
\section{Comparison of Different Influence Analysis Methods} \label{appcompdiam}
Our investigation of Dynamic Influence Tracker (DIT) reveals several important insights about the nature of sample influence in deep learning. We first present a comprehensive comparison of different sample influence analysis methods:

\begin{table*}[h]
	\centering
	\caption{Comprehensive Comparison of Different Sample Influence Analysis Methods}
	\label{tab:keyfindings_comparison}
	\begin{tabular}{lllll}
		\toprule
		\textbf{Aspect} & \textbf{Full-DIT} & \textbf{First-Epoch DIT} & \textbf{Influence Functions} & \textbf{Leave-One-Out} \\
		\midrule
		Space Complexity & $O(T(|S_t| + p))$ & $O(E(|S_t| + p))$ & $O(p^2)$ & $O(p)$ \\
		Time Complexity & $O(T|S_t|p)$ & $O(E|S_t|p)$ & $O(p^3 + Np^2)$ & $O(NT|S_t|p)$ \\
		Robustness to Non-convergence & Yes & Yes & No & No \\ 
		Robustness to Non-convexity & Yes & Yes & No & Yes \\ 
		Robustness to Global Optimality & Yes & Yes & No & Yes \\ 
		Adaptability & Flexible & Flexible & Static & Static \\ 
		\bottomrule
		\multicolumn{3}{l}{\small $T$ = total steps, $E$ = steps per epoch, $p$ = parameters, $|S_t|$ = batch size} \\
	\end{tabular}
\end{table*}

This comparison highlights several key findings:
\begin{enumerate}
	\item Our results show that sample influence is not static but evolves significantly throughout the training process. The identification of four distinct influence patterns (Stable Influencers, Early Influencers, Late Bloomers, and Highly Fluctuating Influencers) challenges the traditional static view of sample importance.
	\item The strong correlation between mid-training and full-training influence measures suggests that influential samples can be identified well before model convergence. This finding has practical implications for efficient training protocols and early intervention strategies.
	\item The consistency of influence patterns across different model architectures for the same dataset suggests that sample influence is more intrinsically tied to data characteristics than model architecture.
\end{enumerate}

\clearpage
\section{Experimental Supplement} 
\subsection{Experimental Setup}\label{appexp}
\paragraph{Datasets}
We employed four diverse datasets spanning various domains and complexities to evaluate the robustness and generalizability of DIT.
\begin{itemize}
	\item \textbf{Adult}~\cite{Dua2019}: A dataset for income prediction containing 48,842 instances with 14 mixed categorical and numerical features. The dataset is preprocessed by handling missing values, normalizing numerical features, and applying one-hot encoding to categorical features. The task is a binary classification of predicting whether income exceeds \$50K/year.
	\item \textbf{20 Newsgroups}~\cite{Lang95}: A text classification dataset. Text data is converted to TF-IDF vectors, and stop words are removed for cleaner feature representation. We focus on binary classification between categories \textit{comp.sys.ibm.pc.hardware} and \textit{comp.sys.mac.hardware}. The task is to classify posts into one of the two hardware categories.
	\item \textbf{MNIST}~\cite{lecun-mnisthandwrittendigit-2010}: A dataset of 70,000 handwritten digit images, each 28x28 pixels in grayscale. Binary classification is conducted between digits \textit{1} and \textit{7}, where the pixel intensities are normalized.
	\item \textbf{EMNIST}~\cite{Cohen2017EMNIST}: An extended MNIST dataset for handwritten letters. The grayscale images are normalized to ensure uniformity in the input space. We focus on binary classification between letters \textit{A} and \textit{B}.
\end{itemize}

\paragraph{Model Architectures}
We evaluated DIT using different model architectures of varying complexity. 
\begin{itemize}
	\item \textbf{Logistic Regression (LR)}: Implemented as a single-layer neural network without hidden layers. The input dimension is flattened to accommodate various input shapes. 
	\item \textbf{Deep Neural Network (DNN)}: The architecture comprises two hidden layers, each with eight units followed by a ReLU activation function. The second layer outputs a single value for binary classification. The input is flattened, similar to logistic regression.
	\item \textbf{Convolutional Neural Network (CNN)}: This architecture is used for image datasets like MNIST and EMNIST. It consists of two convolutional layers, with 32 and 64 filters, respectively, each followed by ReLU activation and max-pooling. The final output from the convolutional layers is flattened and passed through a linear layer to output a binary classification value.
\end{itemize}
For non-image data like Adult and 20 Newsgroups, the input is a vector, while image data like MNIST and EMNIST is reshaped into a single dimension for LR and DNN models. The CNN processes image data in its original 2D format. All these models output a single value and use binary cross-entropy loss with logits for classification, with input/output dimensions adapted to each dataset.

\paragraph{Evaluation Metrics}\label{appmetrics}
To comprehensively evaluate the performance of DIT, we employed a suite of statistical metrics, each capturing different aspects of the relationship between compared methods:
\begin{itemize}
	\item \textbf{Pearson Correlation Coefficient ($r$)}~\cite{pearson1895vii}: The Pearson correlation coefficient measures the linear relationship between two variables. For two sets of data, X and Y, it is calculated as:
	$$r = \frac{\sum_{i=1}^{n} (X_i - \bar{X})(Y_i - \bar{Y})}{\sqrt{\sum_{i=1}^{n} (X_i - \bar{X})^2 \sum_{i=1}^{n} (Y_i - \bar{Y})^2}}$$
	where $\bar{X}$ and $\bar{Y}$ are the means of X and Y respectively, and $n$ is the number of samples. This metric is valuable for identifying direct proportional or inversely proportional relationships within the data.
	$r$ ranges from -1 to 1, where 1 indicates a perfect positive linear relationship, -1 indicates a perfect negative linear relationship, and 0 indicates no linear relationship.
	\item \textbf{Spearman's Rank Correlation Coefficient ($\rho$)}~\cite{spearman1961proof}: Spearman's rank correlation assesses monotonic relationships by comparing the rank orders of samples:
	$$\rho = 1 - \frac{6 \sum_{i=1}^{n} d_i^2}{n(n^2 - 1)}$$
	where $d_i$ is the difference between the ranks of corresponding values $X_i$ and $Y_i$, and $n$ is the number of samples.
	$\rho$ ranges from -1 to 1, with values close to 1 or -1 indicating strong monotonic relationships (positive or negative, respectively) and values close to 0 indicating weak monotonic relationships.
	
	\item \textbf{Kendall's Tau ($\tau$)}~\cite{kendall1938new}: Kendall's Tau evaluates ordinal relationships by measuring the number of concordant and discordant pairs:
	$$\tau = \frac{2(n_c - n_d)}{n(n-1)}$$
	where $n_c$ is the number of concordant pairs, $n_d$ is the number of discordant pairs, and $n$ is the total number of pairs. 
	$\tau$ ranges from -1 to 1, with 1 indicating perfect agreement between two rankings, -1 indicating perfect disagreement, and 0 indicating no relationship.
	
	\item \textbf{Jaccard Similarity ($J$)}~\cite{jaccard1912distribution}: The Jaccard similarity coefficient compares the overlap between the top 30\% of influential points as determined by different methods:
	$$J(A, B) = \frac{|A \cap B|}{|A \cup B|}$$
	where $A$ and $B$ are the sets of top 30\% influential points identified by different methods. 
	$J$ ranges from 0 to 1, with 1 indicating perfect overlap between the sets and 0 indicating no overlap.
\end{itemize}
By capturing linear relationships (Pearson), monotonic relationships (Spearman), ordinal relationships (Kendall's Tau), and set-based similarities (Jaccard), we ensure a multifaceted evaluation of influence estimation methods. 

To ensure transparency and reproducibility, all code, including detailed hyperparameter settings and training procedures, is available on our GitHub repository. This repository contains scripts and configuration files that define the exact setup for each model used in our experiments, encompassing learning rates, batch sizes, regularization strategies, and any other model-specific training details.

\subsection{Sample Influence Dynamics Methodology} \label{apppattern}
The methodology for analyzing sample influence dynamics consists of several key steps.
\begin{enumerate}
	\item \textbf{Sampling and Influence Tracking}: We randomly select 256 training samples and track their influence, measured as loss change via LOO, over 20 epochs of training. This fine-grained sampling provides detailed influence trajectories for each point.
	\item \textbf{ Standardization and Trend Analysis}: We standardize the influence values using StandardScaler to normalize the value across different epochs. For each sample, a linear regression is performed on its standardized influence values over time. The slope of this regression line indicates the overall trend direction (increasing or decreasing influence). The p-value of the regression determines whether this trend is statistically significant.
	\item \textbf{Adaptive Pattern Categorization}: Each sample is categorized based on its statistical properties, including
	a) Trend significance (determined by the p-value)
	b) Trend direction (positive or negative slope)
	c) Standard deviation of influence values (a measure of fluctuation).
	\item \textbf{Pattern Analysis}: We calculate the proportion of samples in each category and compute the centroid of each category by averaging the standardized influence values of all points within that category.
\end{enumerate}

\subsection{Identification of Training Stages} \label{appstages}
To identify stages in the training process, we utilized the following method:
\begin{enumerate}
	\item \textbf{Modeling Loss Trajectory}: We analyzed the loss trajectory across epochs by fitting an exponential decay model. This approach helps to smooth out fluctuations and emphasize underlying trends in the training loss.
	\item \textbf{Residual Calculation}: Residuals were computed as the differences between the actual loss values and the values predicted by the exponential model. These residuals highlight where the actual training deviates from the predicted trend.
	\item \textbf{Change Point Detection}: We identified peaks in the absolute residuals as change points. A minimum distance criterion was applied to ensure these change points were evenly distributed across the training timeline.
	\item \textbf{Stage Segmentation}: Based on the identified change points, the training process was divided into three stages: early, middle, and late. 
\end{enumerate}

\end{document}